\documentclass{article} % For LaTeX2e
\usepackage{iclr2027_conference,times}
\iclrfinalcopy
\usepackage{multirow}
\usepackage[T1]{fontenc}
\usepackage[draft]{hyperref}
\usepackage{url}
\usepackage{booktabs}
\usepackage{amsfonts}
\usepackage{bbm}
\usepackage{nicefrac}
\usepackage{microtype}
\usepackage[table]{xcolor}
\usepackage{adjustbox}
\usepackage{adjustbox}
\usepackage{amsmath}
\usepackage{amssymb}
\usepackage{amsthm}
\usepackage{graphicx}
\usepackage{subcaption}
\usepackage{arydshln}
\usepackage[ruled,vlined]{algorithm2e}
\newtheorem{lemma}{Lemma}

\usepackage{wrapfig}
\usepackage{bbm}
\theoremstyle{definition}
\newtheorem{definition}{Definition}
\theoremstyle{plain}
\newtheorem{proposition}{Proposition}

\newcommand{\namelong}{SmoothOperator} 
\newcommand{\name}{SmoothOP} 
\newcommand{\quantity}{prominence} 
\newcommand{\Quantity}{Prominence} 

\newcommand{\supconmix}{ConOSR}
\newcommand{\supcon}{SupCON}
\newcommand{\sphormix}{SpHOR}
\newcommand{\sphor}{SpHOR-Lite}

\definecolor{gainred}{HTML}{0070C0}
\newcommand{\gaincolor}{blue}

\newcommand{\meanfield}{Batch mean}
\newcommand{\reverse}{Reversed Assignment}

\newcommand{\observerclassifier}{Observer classifier}

\title{\namelong: Enhancing Representations for Fine-grained Open-set Recognition via Modulated Label Smoothing}

\iclrfinalcopy

\author{Thiru Thillai Nadarasar Bahavan, Yu Xia, Sachith Seneviratne \& Saman Halgamuge \\
The University of Melbourne \\
Parkville, VIC, Australia \\
\texttt{bahavant@student.unimelb.edu.au},  \\ 
\texttt{\{yu.xia8, sachith.seneviratne,saman.halgamuge\}@unimelb.edu.au}
}

\iclrfinalcopy % Uncomment for camera-ready version, but NOT for submission.

\begin{document}
\pagenumbering{gobble}
\maketitle

Open Set Recognition (OSR) aims to enable models to accurately classify known classes while rejecting samples from unseen classes. A key challenge in OSR lies in the inability to model the unbounded
distribution of unknown classes during training, often leading to the misclassification of samples from these classes. Rather than modeling unknowns, recent work shapes the feature space so that known classes are compact and well separated, and spherical representation learning methods have achieved strong results this way. Label smoothing has been identified as one of the key drivers of this success, yet it applies the same coefficient to every training sample, regardless of how well each sample is already embedded. We show that the spherical representation learning objectives used in OSR share a single alignment--uniformity structure in which labels enter only through the alignment term. Label smoothing therefore acts as an alignment dial, and a fixed coefficient sets this dial to the same value for every sample. We propose a plug-in, SmoothOperator (SmoothOP), which sets the smoothing coefficient of each sample from its \textbf{prominence}, an embedding-space signal measuring how clearly the sample's own class stands out against its strongest competing class. Our method integrates into four existing spherical representation learning methods at minimal training overhead. SmoothOP assigns strong smoothing to samples with high prominence, which reduces their alignment and relaxes their pull. Conversely, samples with low prominence receive reduced smoothing, which translates to increased alignment, pulling them to the class center. Empirically, SmoothOP yields more compact class clusters and, despite modifying only the alignment term, better dispersed ones, which together improve OSR performance. On the Semantic Shift Benchmark, SmoothOP-augmented variants generally outperform their base objectives across datasets, degrees of semantic shift, and OSR post-processors, with gains of up to 4.7\% in AUROC, OSCR, and closed-set accuracy.

\section{Introduction}

Traditional machine learning classifiers typically operate under a closed-set assumption, where the model assumes all test samples belong to one of the classes encountered during training. In real-world scenarios, however, this assumption often fails as models frequently encounter ``unknown'' classes not present in the initial training set. For instance, a diagnostic tool trained to recognize five specific skin pathologies might erroneously force a sixth, unseen condition into one of its existing categories. Open Set Recognition (OSR) mitigates this risk by enabling systems to correctly identify and label samples from novel classes as ``unknown'' while preserving high classification performance on the original known data \citep{scheirer2013open}.

The main challenge in OSR is that, by definition, unknown classes are not available during training; yet, the model must still reject them during inference. To tackle this problem, earlier OSR literature has followed two main approaches. One approach is to explicitly represent or synthesize the unknown~\citep{bendale2016towards, Zhou2021LearningPF, chen2020learning, chen2021adversarial, Neal_2018_ECCV, kong2022opengan}. Another approach models the known classes with reconstruction-based or generative models~\citep{oza2019c2ae, sun2020open, huang2023class} (Sec.~\ref{sec:relatedwork}). However, \citet{vaze2022openset} showed that standard closed-set classifiers can match or exceed dedicated OSR methods when their closed-set performance is optimized. These results stem from the training procedure rather than from any OSR-specific component, suggesting that the learned feature representation plays a central role in open-set performance~\citep{Dietterich2022FamiliarityHypothesis, sun2022out, sphor}. In terms of \citet{scheirer2013open}'s open-set risk framework, this translates to feature geometry: known classes should be well separated and occupy a compact volume in the feature space, thereby reserving the remaining regions for unknowns.

Building on this, recent OSR methods explicitly shape the feature geometry through representation learning. A model is first trained with a spherical representation learning objective, either tailored to OSR~\citep{conosr, sphor} or general-purpose~\citep{khosla2020supervised}, and a classifier is then fine-tuned on top of the learned representation for classification and unknown rejection. Label smoothing has been identified as one of the performance drivers during the representation learning stage~\citep{sphor,conosr}, yet it applies a single fixed coefficient to every training sample. To understand how smoothing acts on these objectives, we decompose them into two terms~\citep{wang2020understanding}: an \emph{alignment} term, which pulls each sample toward its own class and makes classes compact, and a \emph{uniformity} term, which spreads samples across the hypersphere and keeps classes apart. We find that these objectives share a common structure: the uniformity term does not depend on the ground-truth labels, and the labels only determine how strongly each sample is pulled toward each class within the alignment term (Prop.~\ref{prop:structure}). Since label smoothing modifies only the labels, it acts purely on alignment. Increasing the smoothing coefficient shifts part of this alignment pull from its own class to the other classes, weakening its alignment while leaving the uniformity term unchanged (Prop.~\ref{prop:dial}). Label smoothing is thus an alignment dial, and a fixed coefficient sets this dial to the same value for every sample.

This motivates \name{}, a plug-in that sets the smoothing coefficient for each sample individually. The coefficient is controlled by \quantity{}, an embedding-space signal measuring how clearly a sample's own class stands out against its strongest competing class. Samples with high \quantity{} are already clearly separated from competing classes, so they receive strong smoothing, which reduces their alignment and helps prevent overfitting. Samples with low \quantity{} lie close to a competing class and receive weak smoothing, which strengthens their alignment and pulls them toward their class center. \name{} integrates into existing spherical representation learning methods with minimal training cost and yields class clusters that are more compact and better dispersed.

In this work, we make the following contributions:
\begin{itemize}
\item We show that the contrastive and spherical objectives used in current OSR methods share a single alignment–uniformity decomposition in which labels enter only through the alignment term, and that label smoothing modifies that term alone.
\item Based on those formal arguments, we propose \name{}, a plug-in that sets the smoothing coefficient for each sample from from \quantity{}, which measures how clearly the sample's own class stands out against its strongest competing class. \name{} integrates into four existing spherical representation learning objectives at minimal training cost. 
\item We empirically validate that our plug-in yields class clusters that are more compact and better dispersed. 
\item On the Semantic Shift Benchmark~\citep{vaze2022openset}, our \name{}-augmented variants generally outperform their base objectives across different datasets, degrees of semantic shift, and OSR post-processors, with gains of up to 4.7\% in AUROC, OSCR, and accuracy.
\end{itemize}

\section{Related Work}
\label{sec:relatedwork}

The literature on Open-Set Recognition (OSR) broadly distinguishes between two families of methods: discriminative and generative~\citep{geng2020recent}. Discriminative methods aim to enhance classifiers to correctly identify known classes while simultaneously flagging inputs from unknown classes; representative examples include OpenMax~\citep{bendale2016towards} and PROSER~\citep{Zhou2021LearningPF}. A related line of discriminative work uses class prototypes as anchors for decision boundaries~\citep{Yang2020Convolutional}, while others define reciprocal points exclusion zones in feature space that bound known class regions to tighten the open-space boundary~\citep{chen2020learning, chen2021adversarial}. A shared weakness of these approaches is their difficulty characterizing the full extent of open space, compounded by their reliance on Euclidean geometry~\citep{Wang2024ExploringDR} (discussed further in Sec~\ref{prem}). Generative approaches instead learn the data distribution or manifold structure and detect unknowns either through synthetic sample generation~\citep{Neal_2018_ECCV, kong2022opengan} or via reconstruction error~\citep{oza2019c2ae, huang2023class, sun2020open}. Some methods blend both paradigms~\citep{perera2020generative}. Despite their effectiveness, generative and hybrid methods carry a high computational cost and can be unreliable when out-of-distribution samples bear close resemblance to known classes~\citep{vaze2022openset, huang2023class}. More recently, representation learning has emerged as a promising direction for OSR~\citep{conosr}. Vaze et al.~\citep{vaze2022openset} demonstrated that well-calibrated logit-based classifiers can rival dedicated OSR systems, attributing OSR gains primarily to stronger closed-set architectures and training procedures rather than OSR-specific design choices. 

\textbf{Semantic Shift vs Covariate Shift} Covariate shift changes the input distribution while keeping label semantics fixed. Semantic shift, the setting of OSR, changes the label space through novel classes while the input domain stays largely unchanged (formal definitions in Appendix~\ref{app:dataset_Details}). To isolate the effects of semantic shift from low-level covariate noise, the Semantic Shift Benchmark (SSB) \citep{vaze2022openset} was proposed\citep{vaze2022openset}. In these settings, unknown classes share significant visual attributes with known classes, requiring the model to learn highly precise semantic boundaries.

\textbf{Post-processors} In addition to training time-based methods, there are new OSR methods which can be integrated after the model is trained, which mainly affect changes in the score functions\citep{postmax,nnguide,sun2022knnood}. This is elaborated in Appendix D.3. While these methods are suitable for improving the model's ability to separate known from unknown data, they don't improve the model's accuracy or classification performance.

\textbf{Fine-grained OSR} Most OSR benchmarks involve coarse-grained category shifts, leaving the fine-grained regime comparatively understudied~\citep{finegrainedosr, vaze2022openset}. Lang et al.~\citep{finegrainedosr} address this gap by introducing a hierarchical adversarial training objective for fine-grained OSR; however, their method requires access to structured label hierarchies such as taxonomic annotations, restricting its applicability to domains where such metadata is readily available.

\section{Problem Formulation and Evaluation Metrics} \label{prem}
To begin, we define a supervised training set for a classification task as 
$\mathcal{D}_{\text{train}} = \{(\mathbf{x}_i, y_i)\}_{i=1}^{N} \subset \mathcal{X} \times \mathcal{C}$, where $\mathcal{X}$ is the image space and $\mathcal{C} =  \{1,\dots,C\}$, with $C :=\lvert \mathcal{C}  \rvert $, is the set of `known' classes. Each discrete class label $y_i \in \mathcal{C}$ admits a vector representation via a one-hot encoding function, yielding:
$\mathbf{y}_i = [y_{i1}, y_{i2}, \dots, y_{iC}]^\top \in \{0,1\}^{C}$, where the element corresponding to the ground-truth class $c = y_i$ is $y_{ic} = 1$ and $y_{ik} = 0$ for $k \neq c$. Both the one-hot encoded label $\mathbf{y}_i$ and any subsequent soft targets lie in the probability simplex $\Delta^{C-1} := \{\mathbf{g}\in\mathbb{R}^C_{\geq 0} : \sum_k g_k = 1\}$.

In the closed-set scenario, the model is evaluated on a test set where the labels are strictly drawn from the training categories:
$\mathcal{D}_{K} = \{(\mathbf{x}_j, y_j)\}_{j=1}^{M_K} \subset \mathcal{X} \times \mathcal{C}$. In this setting, the model returns a distribution over known classes as $p(y \mid \mathbf{x}, y \in \mathcal{C})$. Conversely, in Open-Set Recognition (OSR), test images may also originate from unseen classes $\mathcal{U}$, where $\mathcal{U}$ denotes the set of unknown classes and $\mathcal{C} \cap \mathcal{U} = \emptyset$.  The OSR-evaluation test set is a union of known and unknown samples: $\mathcal{D}_{\text{Open}} = \mathcal{D}_{K} \cup \mathcal{D}_{U}$. 
The unknown component is defined as:
$\mathcal{D}_{U} = \{(\mathbf{x}_j, y_j)\}_{j=1}^{M_U} \subset \mathcal{X} \times \mathcal{U}$.

\noindent \textbf{Scoring Function.}  In the open-set setting, the model provides both a class distribution $p(y \mid \mathbf{x}, y \in \mathcal{C})$ and a scoring function $\mathcal{S}(\mathbf{x})$ to indicate whether or not the test sample belongs to any of the known classes. The final prediction $\hat{y}$ is determined by an operational threshold $\theta$:

\begin{equation}
\label{eq:osr-pred}
\hat{y} = 
\begin{cases} 
\text{Known class given via,} \operatorname*{argmax}_{y \in \mathcal{C}} \, p(y \mid \mathbf{x}, y \in \mathcal{C}) & \text{if } \mathcal{S}(\mathbf{x}) \geq \theta, \\ 
\text{Unknown class} & \text{if } \mathcal{S}(\mathbf{x}) < \theta .
\end{cases}
\end{equation}

\section{Methodology}
\label{sec:method}

\citet{scheirer2013open} point out that the key to optimizing OSR performance is to simultaneously reduce the empirical classification risk on known-class data and the open-space risk, the risk of assigning a known class to regions in the representation space that lie far from any regions associated with known-class training data. A representation space with compact, well-separated class clusters satisfies this property and is therefore the desired geometry. Representation-based OSR methods enforce this geometry through a first-stage objective $\mathcal{L}_F$ that shapes the latent space (Table~\ref{tab:align_unif}).

\begin{wrapfigure}{r}{0.4\linewidth}
    \vspace{-\intextsep}
    \centering
    \includegraphics[width=\linewidth]{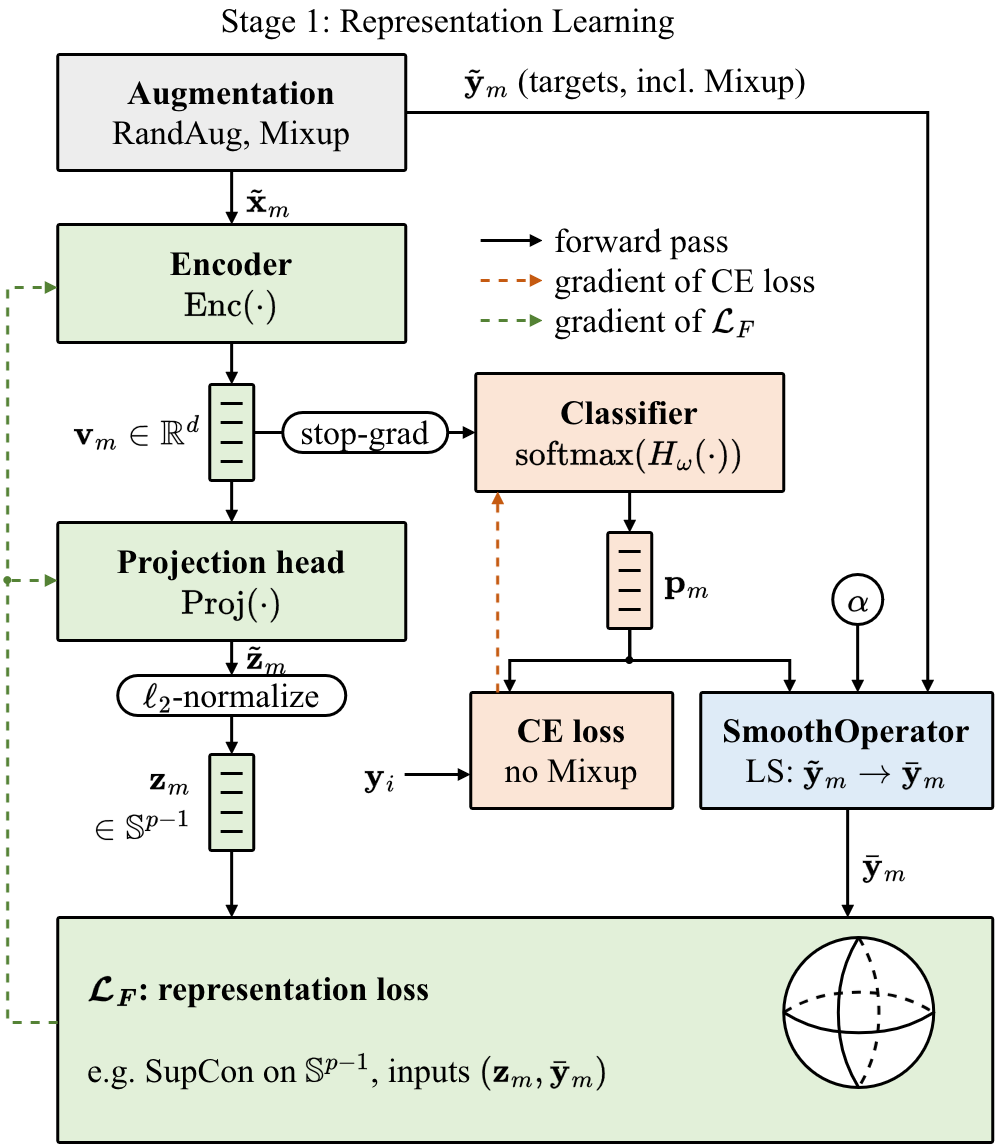}
    \caption{\textbf{Overview of the Proposed Framework.} A three-stage process in which an input sample passes through an augmentation pipeline before reaching the feature encoder. An Observer Classifier estimates \quantity{} from gradient-detached features to modulate label smoothing in the representation learning objective.}
    \label{fig:mainview}
\end{wrapfigure}  %{sec/images/main_method.png}

Concretely, OSR methods based on representation learning follow a three-stage pipeline. The first two stages are inherited from, and theoretically grounded in, self-supervised learning\citep{grill2020bootstrap, contra}. 
In the first stage, an encoder is trained to produce geometrically structured representations, guided by $\mathcal{L}_F$. During the second stage, a lightweight classifier is trained atop the frozen encoder to map representations to a class distribution $p(y \mid \mathbf{x}, y \in \mathcal{C})$, a procedure referred to as linear probing.  In the third stage, for the specific objective of OSR, which is different from closed-set classification, the inference phase comprises two steps: (a) computing a scoring function Eq.~\eqref{eq:osr-pred} $S(x)$ to detect unknown samples, typically via post-processors, and (b) classifying samples identified as known with the closed-set posterior. We intervene in the first stage with a plug-in, a form of sample-adaptive label smoothing algorithm, that plugs into the loss function and upgrades existing representation
objectives into stronger variants. Section~\ref{stage1} will detail the general learning framework, Section~\ref{sec:objectives} will give the theoretical motivation behind our method, and Section~\ref{sec: engineering} will detail the engineering behind the plug-in.

\subsection{General Representation Learning Framework}
\label{stage1}

Across OSR representation learning, the main differences are in how the loss function is constructed. So we first introduce a general framework and elaborate on the loss function design. This section will describe it from the point of view of data flow. 

\textbf{Stage 1:} \textit{RandAugment.} Every input batch $\{\mathbf{x}_i, \mathbf{y}_i\}_{i=1}^{B}$ is pre-processed using the RandAugment protocol \citep{cubuk2020randaugment} to produce $\{\mathbf{x'}_i, \mathbf{y}_i\}_{i=1}^{B}$.  Two label-altering operations then produce the soft targets central to our framework.

\textit{Label Smoothing.} Consistent with established baselines \citep{vaze2022openset}, we replace the one-hot target $\mathbf{y}_i$ with a smoothed target via the map $\mathrm{LS}_{\varepsilon} : \Delta^{C-1} \to \Delta^{C-1}$,
\begin{equation}
\mathbf{y}_i' := \mathrm{LS}_{\varepsilon}(\mathbf{y}_i) = (1 - \varepsilon)\,\mathbf{y}_i + \frac{\varepsilon}{C} \cdot \mathbf{1}, \qquad \varepsilon \in [0,1],
\label{eq:ls3}
\end{equation} 
where $\varepsilon$ is the label smoothing coefficient, and $\mathbf{1}$ is an all-ones vector.

\textit{Mixup.} For objectives that employ it (e.g., ConOSR, SpHOR), Mixup regularization \citep{zhang2018mixup} generates virtual samples.
\begin{equation}
\hat{\mathbf{x}}_{i} = \lambda \mathbf{x}'_{i} + (1 - \lambda) \mathbf{x}'_{\pi(i)}, \quad 
\hat{\mathbf{y}}_{i} = \lambda \mathbf{y}'_{i} + (1 - \lambda) \mathbf{y}'_{\pi(i)},
\end{equation}
where $\lambda \sim \mathrm{Beta}(2,2)$ is drawn per batch, $\pi(i)$ is a a random permutation of the batch. 

Both these operations map the simplex to itself, so the target remains a valid distribution. The final augmented batch, denoted as $\{(\tilde{\mathbf{x}}_m, \tilde{\mathbf{y}}_m)\}_{m=1}^{2B} :=
\{(\mathbf{x}'_i, \mathbf{y}'_i)\}_{i=1}^{B} \cup
\{(\hat{\mathbf{x}}_i, \hat{\mathbf{y}}_i)\}_{i=1}^{B}$, provides a unified set of original and synthetic samples on which the model is trained.  

\textit{Training, Encoding and Projection.} An encoder network $Enc(\cdot)$ is employed to map each sample into a $d$-dimensional representation $\mathbf{v}_m = \mathrm{Enc}(\tilde{\mathbf{x}}_m)$. The representations are mapped through a projection head $Proj(\cdot)$ to yield $p$-dimensional projections $\tilde{\mathbf{z}}_m = Proj(\mathbf{v}_m)$. The $\tilde{\mathbf{z}}_m$ are $L_2$ normalized to produce $\mathbf{z}_m$.  For brevity, we denote the composite mapping by $f$, i.e., \
$f(\mathbf{x}) \triangleq \mathrm{Proj}(\mathrm{Enc}(\mathbf{x})) \,/\,
\lVert \mathrm{Proj}(\mathrm{Enc}(\mathbf{x})) \rVert_2$, such that
$\mathbf{z}_m = f(\tilde{\mathbf{x}}_m) \in \mathbb{S}^{p-1}$. The encoder and projection head are trained jointly to minimize
the representation objective $\mathcal{L}_F$, computed on the normalized
projections. This reflects the standard contrastive learning protocol\citep{khosla2020supervised}.
We use four representation objectives: \textbf{ConOSR, SpHOR, Supervised Contrastive Learning with Label Smoothing (\supcon{}) and SpHOR without Mixup (\sphor{})}~\citep{khosla2020supervised,conosr,sphor}. The algorithm details and how \name{} is fused with it are detailed in the Appendix~\ref{app:repl_integration}.

\textbf{Stage 2: Linear Probing.}\label{stage2}
Once Stage 1 is complete, the projection head is discarded and the encoder $\text{Enc}(\cdot)$ is frozen. To enable class-level predictions for the known set $\mathcal{C}$, we append a linear classification head $H_{\omega}(\cdot)$ to the fixed $d$-dimensional embeddings $\mathbf{f}$. During this stage, we optimize the weights $\omega$ using a standard cross-entropy loss over the known classes: $\mathcal{L}_{\mathrm{CE}}(\omega)
= -\frac{1}{N}\sum_{i=1}^{N}\sum_{k\in\mathcal{C}}
y_{ik}\,\log\big[\mathrm{Softmax}\big(H_{\omega}(\mathbf{v}_i)\big)\big]_{k}$. No label smoothing in this stage.

\textbf{Stage 3: Open-Set Inference.}\label{stage3}
With $(\mathrm{Enc}, H_{\omega})$ frozen, the posterior in
Eq.~\eqref{eq:osr-pred} is computed as
$p(y \mid \mathbf{x}, y \in \mathcal{C}) =
\mathrm{Softmax}(H_{\omega}(\mathbf{v}))$,
$\mathbf{v}=\mathrm{Enc}(\mathbf{x})$, and the score
$\mathcal{S}(\mathbf{x})$ as a training-free post-processor on
$\mathbf{v},~p(y \mid \mathbf{x}, y \in \mathcal{C})$;
prediction then follows Eq.~\eqref{eq:osr-pred}.

\subsection{Representation Objectives: Alignment and Uniformity}
\label{sec:objectives}
Contrastive representation losses share a single structure. In this framework, a reference sample (the \emph{anchor}) is attracted to semantically similar samples (\emph{positives}) and repelled from semantically distinct ones (\emph{negatives}). The objectives differ in their choice of positives and negatives, the magnitude of attraction or repulsion and the inclusion of auxiliary regularization terms. While self-supervised methods focus on instance discrimination, supervised and
OSR-based methods extend this machinery to enforce class-level discrimination. The strong empirical performance of contrastive learning approaches is explained by two properties of projections: \textit{alignment}
and \textit{uniformity} \citep{wang2020understanding}.

Alignment ensures that for any reference sample acting as an \textit{anchor}, its semantic \textit{positives} are mapped to close proximity in the projection space, making the projections robust to noise. 

\begin{definition}[Alignment]
\label{def:align}
Formally, given a distribution of positive pairs $p_{\text{pos}}$, alignment is defined via expected distance:
\begin{equation}
    \mathcal{L}_{\text{align}}(f; \gamma) \triangleq \mathbb{E}_{(x,x^{+}) \sim p_{\text{pos}}} \left[ \| f(x) - f(x^{+}) \|_2^\gamma \right] = 2 + 2 \mathbb{E}_{(x,x^{+}) \sim p_{\text{pos}}} \left[ -f(x)^\top f(x^{+}) \right], \quad \gamma = 2.
\end{equation}
\end{definition}

Uniformity, on the other hand, dictates that the projections of samples drawn from the data distribution should be distributed roughly uniformly across the unit hypersphere $\mathbb{S}^{p-1}$, preventing representation collapse and retaining maximal information. 
\begin{definition}[Uniformity]
\label{def:unif}
For samples drawn i.i.d.\ from the data distribution $p_{\text{data}}$,
\begin{equation}
    \mathcal{L}_{\text{uniform}}(f; t) \triangleq \log \mathbb{E}_{x,x' \stackrel{\text{i.i.d.}}{\sim} p_{\text{data}}} \left[ e^{-t \| f(x) - f(x') \|_2^2} \right] = -2t + \log \mathbb{E}_{x,x' \stackrel{\text{i.i.d.}}{\sim} p_{\text{data}}} \left[ \exp\left(2t \, f(x)^\top f(x')\right) \right]
\end{equation}
\end{definition}

\begin{table}[t]
\centering
\small
\setlength{\tabcolsep}{5pt}
\renewcommand{\arraystretch}{2.1}
\begin{tabular}{@{}ccc c@{}}
\toprule
\textbf{Method} & \textbf{Target $S_{i\cdot}$} &
\textbf{Alignment $\mathcal{L}^{(i)}_{\mathrm{align}}$} &
\textbf{Uniformity $\mathcal{L}^{(i)}_{\mathrm{unif}}$} \\
\midrule
\supcon
& $S_{ij} = \dfrac{s(\tilde{\mathbf{y}}_i, \tilde{\mathbf{y}}_j)}{\sum_{n \neq i} s(\tilde{\mathbf{y}}_i, \tilde{\mathbf{y}}_n)}$
& $\dfrac{1}{2\tau} \sum\nolimits_{j \neq i} S_{ij}\, \lVert \mathbf{z}_i - \mathbf{z}_j \rVert_2^{2}$
& $\log \sum\nolimits_{n \neq i}\exp\!\left(-\dfrac{\lVert \mathbf{z}_i - \mathbf{z}_n \rVert_2^{2}}{2\tau}\right)$ \\
\supconmix
& $S_{ij} = \dfrac{s(\tilde{\mathbf{y}}_i, \tilde{\mathbf{y}}_j)}{\sum_{n \neq i} s(\tilde{\mathbf{y}}_i, \tilde{\mathbf{y}}_n)}$
& $\dfrac{1}{2\tau} \sum\nolimits_{j \neq i} S_{ij}\, \lVert \mathbf{z}_i - \mathbf{z}_j \rVert_2^{2}$
& $\log \sum\nolimits_{n \neq i}\exp\!\left(-\dfrac{\lVert \mathbf{z}_i - \mathbf{z}_n \rVert_2^{2}}{2\tau}\right)$ \\
\sphor
& $S_{ik} = \dfrac{\tilde{y}_{ik}}{\sum_{j=1}^{C} \tilde{y}_{ij}}$
& $\dfrac{1}{2\tau} \sum\nolimits_{k=1}^{C} S_{ik}\, \lVert \mathbf{z}_i - \boldsymbol{\mu}_k \rVert_2^{2}$
& $\log \sum\nolimits_{j=1}^{C}\exp\!\left(-\dfrac{\lVert \mathbf{z}_i - \boldsymbol{\mu}_j \rVert_2^{2}}{2\tau}\right)$ \\
\sphormix
& $S_{ik} = \dfrac{\tilde{y}_{ik}}{\sum_{j=1}^{C} \tilde{y}_{ij}}$
& $\dfrac{1}{2\tau} \sum\nolimits_{k=1}^{C} S_{ik}\, \lVert \mathbf{z}_i - \boldsymbol{\mu}_k \rVert_2^{2}$
& $\log \sum\nolimits_{j=1}^{C}\exp\!\left(-\dfrac{\lVert \mathbf{z}_i - \boldsymbol{\mu}_j \rVert_2^{2}}{2\tau}\right)$ \\
\bottomrule
\end{tabular}
\caption{Alignment--uniformity decomposition of the target-dependent loss of
each representation learning objective.  The target distribution $S_{ab}$, the sole carrier of the
label smoothing degree, appears only in the alignment term, while the
uniformity term is independent of the targets. This is detailed in Appendix~\ref{supp:au_decomposition}.    Here $s(\cdot,\cdot)$ is cosine similarity and $\tilde{\mathbf{y}}_i$ is the label-smoothed target
}.
\label{tab:align_unif}
\end{table}

We first decompose all four objectives into alignment and uniformity, as shown in Table~\ref{tab:align_unif}. This yields the following proposition. 

\begin{proposition}[Shared structure]
\label{prop:structure}
Let $\mathcal{L}_i$ be the target-dependent loss of any objective in
Table~\ref{tab:align_unif}, with attraction weights $S_{i\cdot}$
satisfying $\sum_{a} S_{ia} = 1$ over the method's contrast set. Then
$\mathcal{L}_i$ decomposes exactly as
$\mathcal{L}_i = \mathcal{L}^{(i)}_{\mathrm{align}} +
\mathcal{L}^{(i)}_{\mathrm{unif}}$, where
$\mathcal{L}^{(i)}_{\mathrm{unif}}$ is independent of the targets
$\tilde{\mathbf{y}}$, and $\mathcal{L}^{(i)}_{\mathrm{align}}$ depends on
the targets only through $S_{i\cdot}$.
\end{proposition}
We leverage this decoupling to our advantage. Because label smoothing acts
solely on the targets $\tilde{\mathbf{y}}_i$, it exclusively modulates
alignment. The table and the proposition are detailed in
Appendix~\ref{supp:au_decomposition}.

\begin{proposition}[Label smoothing as an alignment dial]
\label{prop:dial}
Under the conditions of Prop.~\ref{prop:structure}, applying \name{} to the
anchor targets leaves $\mathcal{L}^{(i)}_{\mathrm{unif}}$ unchanged and
modifies $\mathcal{L}^{(i)}_{\mathrm{align}}$ only through $S_{i\cdot}$.
For the \sphor{}/\sphormix{}, increasing $\varepsilon_i$ transfers
attraction mass $\varepsilon_i(1-1/C)$ from the ground-truth label embedding
uniformly onto the others; for the \supcon{}/\supconmix{} with
one-hot contrast labels, each positive's weight decreases and each
negative's weight increases strictly in $\varepsilon_i$.
\end{proposition}
Proof in Appendix~\ref{app:repl_integration}
(Props.~\ref{prop:supcon},~\ref{prop:monotone}).

\subsection{\Quantity{} and Adaptive Label Smoothing}\label{sec: engineering}
Our approach consists of two main components: a per-sample score that quantifies how a sample's own class stands out against its strongest competing class, and a bounded mapping function that translates this score into an individual smoothing coefficient $\varepsilon_i$.

To compose a signal related to the structure of the embedding geometry, we simultaneously train an \textbf{\observerclassifier{}} $H_{\omega}(\cdot)$ (A linear head, which is reused during the Second Stage) using cross-entropy on the gradient-detached embeddings $\mathbf{v_m}$. (so it doesn't interfere with the training dynamics and prevents circularity). This supplies the evolving predicted class probabilities $\mathbf{p}^{(i)} = [p_{1}^{(i)}, \dots, p_{C}^{(i)}]^\top$.

\begin{definition}[\Quantity{}]
\label{def:align2}
Let $\mathbf{x}_i$ be a sample with ground-truth class $c$, let $p_{c}^{(i)}$ be the observer's predicted probability for $c$, and let $\hat{p}^{(i)} = \max_{k \in \mathcal{C} \setminus \{c\}} p_{k}^{(i)}$ be the highest probability among the competing classes. We define the probability margin $u_{i} \in [-1,1]$ as
\begin{equation}
u_{i} := \hat{p}^{(i)} - p_{c}^{(i)}.
\end{equation}
A sample is \emph{prominent} when its own class clearly stands out against its strongest competitor, i.e., when $u_i$ is low; its \quantity{} decreases as $u_i$ increases. Conversely, a sample with $u_i$ near or above zero is ambiguous: it is similar to its own class but also to at least one competing class, and can easily be misclassified.
\end{definition}

Note that this margin has two modes of behavior depending on whether the classifier classifies the sample correctly or not:
\begin{itemize}
    \item \textit{Correct Prediction:} the predicted class is $j$, so $\hat{p}^{(i)} \leq p_{j}^{(i)}$
and $u_{i}$ are negative. As the model grows more confident, it assigns more weight to the true class, and $u_{i}$ decreases towards $-1$. Effectively, the
sample has high \quantity{}.
\item \textit{Incorrect Prediction:} the maximum probability belongs to a class other than $j$,
so $p_{j}^{(i)} < \hat{p}^{(i)}$ and $u_{i}$ is positive. A margin near zero means the model
cannot separate the true class from a competing one, which is genuine ambiguity. A large
positive margin is a confidently wrong prediction. Effectively, as $u_{i}$ \emph{increases},
the sample has low \quantity{}.
\end{itemize}

If compact, well-separated clusters are the desired geometry, samples
that could be confused with another class are where it breaks down:
they sit at the edge of their cluster, close to a competing one, and
are the most likely to be misclassified. These are exactly the samples
with low \quantity{}. Samples with high \quantity{} already fit this
geometry, so it is those with low \quantity{} that need a stronger
pull toward their class. This motivates our algorithm design, which is presented below. 

To make label smoothing adaptive, we convert the smoothing coefficient into a per-sample basis:

\begin{equation}
\mathbf{y_i}' := \mathrm{\name{}}(\mathbf{y}_i) = (1 - \varepsilon_i)\,\mathbf{y}_i + \frac{\varepsilon_i}{C} \cdot \mathbf{1}, \qquad \varepsilon \in [0,1],
\label{eq:ls1}
\end{equation} 
We modulate this per-sample smoothing coefficient based on the sample's corresponding \quantity{} as follows:
\begin{equation}
\varepsilon_i \;=\; \varepsilon \cdot 2\operatorname{sigmoid}(-|\alpha| u_i)
\label{eq:lssubeq}
\end{equation}

Here, $\varepsilon$ is the original default label smoothing coefficient, except that it is modulated by the \quantity{}. $\alpha$ is the modulating strength.

In this setup, the degree of smoothing reduces with \quantity{}. For incorrectly classified, difficult samples that have high $u$, the degree of smoothing is less. However, for correctly classified, easy samples, the degree of smoothing is high. Thus, the strength of the correction is based on the $\alpha$ value.

\begin{proposition}[Properties of the sigmoid modulator]
\label{prop:sigmoid}
Let $\varepsilon_i = 2\varepsilon\,\sigma(-|\alpha| u_i)$ with
$\sigma(t)=(1+e^{-t})^{-1}$, $\varepsilon\in(0,1)$, $\alpha\ge 0$. Then:
(i)~\emph{Boundedness:} $\varepsilon_i\in(0,2\varepsilon)$ for all
$u_i\in\mathbb{R}$; in particular $\alpha=0$ yields
$\varepsilon_i\equiv\varepsilon$, recovering fixed smoothing
(Eq.~\ref{eq:ls1}).
(ii)~\emph{Monotonicity:} for $\alpha>0$, $\varepsilon_i$ is strictly
decreasing in $u_i$.
\end{proposition}

\label{section:epsblender}
\label{mixupstage} While \quantity{} calculation is straightforward for individual samples, characterizing the \quantity{} of a synthetic sample generated via Mixup \citep{zhang2018mixup} presents a unique challenge, because it has two target classes. We hypothesize that the \quantity{} of a composite sample is a function of the constituent samples'  \quantity{}, propagated through the mixing manifold. Consistent with the linear interpolation principle of Mixup, we define the composite  \quantity{} $U(\mathbf{\hat{x}}_{k})$ as a weighted combination of the individual \quantity{} (We retrieve this via the indexes of the pre-mixed samples):
The constituent margins are read from the \emph{unmixed} forward pass and retrieved through
the mixing indices. 
\begin{equation}
\text{u}^{\lambda}_{i,j} = \lambda \cdot \text{u}_{i} + (1 - \lambda) \cdot \text{u}_{j} \qquad  \varepsilon^{\lambda}_{i,j} \;=\; 2\varepsilon \cdot
   \operatorname{sigmoid}\!\big(-|\alpha|\,\text{u}^{\lambda}_{i,j}\big) \in (0,\, 2\varepsilon),
\end{equation}
This composite margin induces a single smoothing coefficient, applied once to the mixed
target:
\begin{equation}
\hat{\mathbf{y}}_{k}^{'} =
  (1 - \varepsilon^{\lambda}_{i,j})\big(\lambda \mathbf{y}_{i} + (1-\lambda)\mathbf{y}_{j}\big)
  + \frac{\varepsilon^{\lambda}_{i,j}}{C} \cdot \mathbf{1},
\label{eq:lsmix1}
\end{equation}
where $\mathbf{y}_{i}, \mathbf{y}_{j}$ are the \emph{one-hot} targets of the constituent
samples. This is empirically ablated in Appendix~\ref{sec:epsblend}.

\section{Experimental Setup} \label{experimentsssection}

\textbf{Datasets.} We evaluate our approach on the three fine-grained datasets in Semantic Shift Benchmark~\citep{vaze2022openset}, which comprises three fine-grained classification datasets: CUB \citep{Wah2011CaltechUCSD}, FGVC-Aircraft \citep{Maji2013FineGrained}, and Stanford Cars \citep{Krause20133D}. Each dataset has an associated unknown class test sets (‘Easy’ and ‘Hard’) based on their semantic similarity to the training classes. Further details on the datasets are in Appendix~\ref{app:dataset_Details}.

\textbf{Hyperparameter Tuning.} Base smoothing $\lambda$ is 0.3 on CUB and 0.2 on Aircraft and Cars. This was selected directly from \cite{vaze2022openset} and we adopt these values without further tuning. We find the optimal $\alpha$ parameter for each dataset--method combination using a three-fold cross-validation on a training split, with images from different fine-grained datasets as validation unknowns. We take the $\alpha$ that maximizes mean OSCR under MSP. We defer the complete sweep grid to Appendix~\ref{app:Hyperparameterdetails}.

\paragraph{Training.} We use a ResNet-50 pretrained on the Places-365 dataset (via Supervised Training) for our experiments.  We optimize with SGD (momentum $0.9$, learning rate $0.05$,
weight decay $10^{-4}$). We train with distributed data parallelism on 4 GPUs (2 nodes $\times$ 2 GPUs), with 128 images per GPU, for 200 epochs. The learning rate follows
a cosine schedule with warm restarts: a 10-epoch linear warmup, then three cosine
cycles of 66 epochs each (two restarts), annealed to $10^{-3}\times$ the base rate.  The contrastive temperature is $0.1$. Feature dimension is 2048. The projector is a $\mathrm{Linear}(2048{\to}2048)$ layer. We augment four baseline representation methods, i.e., SpHOR\citep{sphor}, SpHOR without Mixup (\sphor), ConOSR \citep{conosr}, and Supervised Contrastive Learning (\supcon) \citep{khosla2020supervised}. All methods were trained over three runs using fixed seeds (0, 1, 2), with results reported as the mean across these trials for methodological rigor. The backbone's classifier head is retrained (with the backbone frozen) for 20 epochs in stage 2 before evaluations.  
Full hyperparameter settings (including Stage 2) are listed in Appendix~\ref{app:Hyperparameterdetails}.

 % The models were trained using Nvidia A100 GPUs with PyTorch.

\textbf{Evaluation and Metrics.} We evaluated all other methods using state-of-the-art post-processors/scoring functions, including MLS\citep{vaze2022openset} , nnGuide ($k=5$) \citep{nnguide}, PostMAX \citep{postmax}, and Maximum Softmax Probability (MSP). OSR performance is evaluated on two axes: how well unknowns can be rejected and how well knowns can be classified. The former is quantified using \textit{AUROC}. The latter is quantified by \textit{Closed-set accuracy}. Since strength on one axis can mask weakness on the other, neither suffices alone: Both axes are jointly evaluated by \textit{Open-Set Classification Rate (OSCR)}~\citep{OSCR}.  A detailed explanation/justification of the metrics is in Appendix~\ref{app:metrics} and post-processors is in Appendix~\ref{appendix:scoringfunction}.

\begin{table*}[!t]
\centering
\scriptsize
\caption{\textbf{Fine-grained OSR performance on the Semantic Shift Benchmark (SCARS, FGVC-Aircraft, and CUB): OSCR and AUROC} Each dataset has an "Easy" and "Hard" unknown class split\citep{vaze2022openset}. We designate OSCR as the primary metric. For each representation learning method (ConOSR\citep{conosr}, SpHOR\citep{sphor}, Supervised Contrastive Learning\citep{khosla2020supervised}) and its augmented variant (+), we evaluate performance across four SOTA scoring functions/postprocessors. Our augmented variants consistently outperform baselines across various metrics and datasets. All methods utilize a ResNet-50 backbone pre-trained on the Places dataset. Gains are in blue. The results are given as the mean $\pm$ sd over three training runs with fixed seeds.} 
\label{tab:osr_main}
\resizebox{\textwidth}{!}{
% generated by paper_build_xval/tables_xval.py -- do not edit
\begin{tabular}{l l rr rr rr rr rr rr}
\toprule
 &  & \multicolumn{6}{c}{OSCR} & \multicolumn{6}{c}{AUROC} \\
\cmidrule(lr){3-8}\cmidrule(lr){9-14}
 &  & \multicolumn{2}{c}{CUB} & \multicolumn{2}{c}{Aircraft} & \multicolumn{2}{c}{Cars} & \multicolumn{2}{c}{CUB} & \multicolumn{2}{c}{Aircraft} & \multicolumn{2}{c}{Cars} \\
\cmidrule(lr){3-4}\cmidrule(lr){5-6}\cmidrule(lr){7-8}\cmidrule(lr){9-10}\cmidrule(lr){11-12}\cmidrule(lr){13-14}
 &  & E & H & E & H & E & H & E & H & E & H & E & H \\

% \midrule
% \multicolumn{14}{l}{\textbf{Reference arms at $\lambda{=}0$ (no smoothing)}\quad\emph{pure SupCon exists on CUB only}} \\
% MSP & CE $\lambda{=}0$ & $59.54$ & $54.49$ & --- & --- & --- & --- & $74.85$ & $66.51$ & --- & --- & --- & --- \\
%  & ARPL $\lambda{=}0$ & $60.13$ & $55.99$ & $70.20$ & $69.39$ & $83.04$ & $75.12$ & $75.04$ & $68.23$ & $78.18$ & $76.54$ & $87.57$ & $78.27$ \\
 
%  & SupCon $\lambda{=}0$ & $59.54$ & $54.49$ & --- & --- & --- & --- & $74.85$ & $66.51$ & --- & --- & --- & --- \\
%  & DEF $\lambda{=}0$ & $60.13$ & $55.99$ & $70.20$ & $69.39$ & $83.04$ & $75.12$ & $75.04$ & $68.23$ & $78.18$ & $76.54$ & $87.57$ & $78.27$ \\
%  & DEF $+$ LS (Baseline) $^{\ddagger}$ & $55.76$ & $53.54$ & --- & --- & --- & --- & $71.07$ & $67.08$ & --- & --- & --- & --- \\
% \addlinespace[2pt]

\midrule

MSP & CE & $58.64\,{\pm}\,0.71$ & $55.75\,{\pm}\,0.44$ & $69.63\,{\pm}\,0.72$ & $67.97\,{\pm}\,0.45$ & $83.28\,{\pm}\,0.36$ & $76.48\,{\pm}\,0.25$ & $73.20\,{\pm}\,1.32$ & $68.42\,{\pm}\,0.24$ & $76.94\,{\pm}\,0.55$ & $75.13\,{\pm}\,1.05$ & $87.54\,{\pm}\,0.33$ & $79.53\,{\pm}\,0.38$ \\
 & \supcon (w/o LS) & $59.54\,{\pm}\,1.30$ & $54.49\,{\pm}\,1.17$ & $69.16\,{\pm}\,1.02$ & $68.46\,{\pm}\,0.82$ & $82.13\,{\pm}\,0.51$ & $75.61\,{\pm}\,0.79$ & $74.85\,{\pm}\,0.87$ & $66.51\,{\pm}\,1.18$ & $77.14\,{\pm}\,1.06$ & $75.53\,{\pm}\,0.72$ & $86.37\,{\pm}\,0.80$ & $78.75\,{\pm}\,0.82$ \\
% & DEF $\lambda{=}0$ & $60.13\,{\pm}\,1.69$ & $55.99\,{\pm}\,1.45$ & $70.20\,{\pm}\,1.12$ & $69.39\,{\pm}\,1.19$ & $83.04\,{\pm}\,1.16$ & $75.12\,{\pm}\,2.56$ & $75.04\,{\pm}\,1.32$ & $68.23\,{\pm}\,1.23$ & $78.18\,{\pm}\,1.56$ & $76.54\,{\pm}\,1.13$ & $87.57\,{\pm}\,0.67$ & $78.27\,{\pm}\,2.31$ \\
\hline
\multicolumn{14}{l}{\textbf{\supcon{}} \quad +\name{}: CUB~$-0.5$, Aircraft~$-4$, Cars~$-3$} \\
MSP & Baseline & $61.42\,{\pm}\,0.87$ & $56.64\,{\pm}\,0.84$ & $70.87\,{\pm}\,1.75$ & $66.93\,{\pm}\,2.34$ & $83.61\,{\pm}\,0.74$ & $75.36\,{\pm}\,1.46$ & $76.65\,{\pm}\,0.45$ & $68.76\,{\pm}\,0.29$ & $80.28\,{\pm}\,0.60$ & $74.82\,{\pm}\,1.34$ & $88.49\,{\pm}\,0.13$ & $78.70\,{\pm}\,1.07$ \\

 & +\name{} & $62.54\,{\pm}\,0.49$ & $57.15\,{\pm}\,0.85$ & $72.08\,{\pm}\,1.12$ & $69.26\,{\pm}\,0.97$ & $83.69\,{\pm}\,0.29$ & $75.60\,{\pm}\,0.91$ & $78.18\,{\pm}\,0.62$ & $69.34\,{\pm}\,0.98$ & $80.38\,{\pm}\,0.75$ & $76.63\,{\pm}\,1.58$ & $88.14\,{\pm}\,0.59$ & $78.69\,{\pm}\,0.90$ \\

 & gain & \textcolor{gainred}{$\mathbf{+1.82\%}^{*}$} & \textcolor{gainred}{$\mathbf{+0.91\%}$} & \textcolor{gainred}{$\mathbf{+1.70\%}$} & \textcolor{gainred}{$\mathbf{+3.49\%}$} & \textcolor{gainred}{$\mathbf{+0.09\%}$} & \textcolor{gainred}{$\mathbf{+0.32\%}$} & \textcolor{gainred}{$\mathbf{+2.00\%}^{*}$} & \textcolor{gainred}{$\mathbf{+0.85\%}$} & \textcolor{gainred}{$\mathbf{+0.12\%}$} & \textcolor{gainred}{$\mathbf{+2.43\%}$} & $\mathbf{-0.40\%}$ & $\mathbf{-0.01\%}$ \\
\addlinespace[2pt]

MLS & Baseline & $62.55\,{\pm}\,1.09$ & $56.16\,{\pm}\,0.71$ & $72.06\,{\pm}\,1.95$ & $67.09\,{\pm}\,2.39$ & $84.07\,{\pm}\,0.77$ & $75.58\,{\pm}\,1.41$ & $79.23\,{\pm}\,0.50$ & $68.65\,{\pm}\,0.09$ & $82.18\,{\pm}\,0.83$ & $75.40\,{\pm}\,1.41$ & $89.30\,{\pm}\,0.27$ & $79.09\,{\pm}\,0.99$ \\
 & +\name{} & $63.43\,{\pm}\,0.55$ & $56.67\,{\pm}\,0.83$ & $73.22\,{\pm}\,1.02$ & $69.21\,{\pm}\,1.13$ & $84.22\,{\pm}\,0.15$ & $75.72\,{\pm}\,0.74$ & $80.20\,{\pm}\,0.61$ & $69.17\,{\pm}\,1.00$ & $82.08\,{\pm}\,0.67$ & $76.88\,{\pm}\,1.77$ & $89.06\,{\pm}\,0.39$ & $78.99\,{\pm}\,0.74$ \\
 & gain & \textcolor{gainred}{$\mathbf{+1.41\%}$} & \textcolor{gainred}{$\mathbf{+0.91\%}$} & \textcolor{gainred}{$\mathbf{+1.60\%}$} & \textcolor{gainred}{$\mathbf{+3.15\%}$} & \textcolor{gainred}{$\mathbf{+0.19\%}$} & \textcolor{gainred}{$\mathbf{+0.19\%}$} & \textcolor{gainred}{$\mathbf{+1.22\%}^{*}$} & \textcolor{gainred}{$\mathbf{+0.76\%}$} & $\mathbf{-0.13\%}$ & \textcolor{gainred}{$\mathbf{+1.97\%}$} & $\mathbf{-0.27\%}$ & $\mathbf{-0.13\%}$ \\
\addlinespace[2pt]
NNGuide & Baseline & $52.05\,{\pm}\,1.53$ & $49.22\,{\pm}\,0.88$ & $68.93\,{\pm}\,2.52$ & $66.19\,{\pm}\,3.04$ & $77.99\,{\pm}\,2.75$ & $73.67\,{\pm}\,1.83$ & $67.08\,{\pm}\,1.30$ & $62.80\,{\pm}\,1.83$ & $77.72\,{\pm}\,1.78$ & $74.24\,{\pm}\,1.93$ & $82.80\,{\pm}\,2.94$ & $78.04\,{\pm}\,1.42$ \\
 & +\name{} & $54.19\,{\pm}\,0.65$ & $50.22\,{\pm}\,2.75$ & $70.75\,{\pm}\,1.97$ & $68.56\,{\pm}\,1.10$ & $79.51\,{\pm}\,1.45$ & $74.05\,{\pm}\,0.47$ & $69.43\,{\pm}\,0.70$ & $63.48\,{\pm}\,3.60$ & $78.99\,{\pm}\,2.13$ & $76.26\,{\pm}\,1.71$ & $84.29\,{\pm}\,1.57$ & $78.14\,{\pm}\,0.36$ \\
 & gain & \textcolor{gainred}{$\mathbf{+4.12\%}$} & \textcolor{gainred}{$\mathbf{+2.03\%}$} & \textcolor{gainred}{$\mathbf{+2.65\%}$} & \textcolor{gainred}{$\mathbf{+3.58\%}$} & \textcolor{gainred}{$\mathbf{+1.96\%}$} & \textcolor{gainred}{$\mathbf{+0.52\%}$} & \textcolor{gainred}{$\mathbf{+3.51\%}$} & \textcolor{gainred}{$\mathbf{+1.07\%}$} & \textcolor{gainred}{$\mathbf{+1.63\%}$} & \textcolor{gainred}{$\mathbf{+2.72\%}$} & \textcolor{gainred}{$\mathbf{+1.81\%}$} & \textcolor{gainred}{$\mathbf{+0.13\%}$} \\
\addlinespace[2pt]
PostMax & Baseline & $63.18\,{\pm}\,1.28$ & $55.85\,{\pm}\,0.78$ & $71.69\,{\pm}\,2.47$ & $64.47\,{\pm}\,0.83$ & $84.01\,{\pm}\,0.76$ & $76.08\,{\pm}\,0.71$ & $80.04\,{\pm}\,0.62$ & $68.36\,{\pm}\,0.20$ & $81.92\,{\pm}\,1.62$ & $72.63\,{\pm}\,0.68$ & $89.18\,{\pm}\,0.35$ & $79.56\,{\pm}\,0.59$ \\

 & +\name{} & $63.75\,{\pm}\,0.68$ & $56.53\,{\pm}\,0.36$ & $72.62\,{\pm}\,1.22$ & $67.51\,{\pm}\,0.79$ & $84.36\,{\pm}\,0.49$ & $75.31\,{\pm}\,0.99$ & $80.62\,{\pm}\,0.36$ & $69.03\,{\pm}\,0.73$ & $81.54\,{\pm}\,1.04$ & $75.08\,{\pm}\,1.19$ & $89.18\,{\pm}\,0.57$ & $78.57\,{\pm}\,1.02$ \\
 & gain & \textcolor{gainred}{$\mathbf{+0.90\%}$} & \textcolor{gainred}{$\mathbf{+1.22\%}$} & \textcolor{gainred}{$\mathbf{+1.30\%}$} & \textcolor{gainred}{$\mathbf{+4.70\%}^{*}$} & \textcolor{gainred}{$\mathbf{+0.42\%}$} & $\mathbf{-1.00\%}$ & \textcolor{gainred}{$\mathbf{+0.72\%}^{\dagger}$} & \textcolor{gainred}{$\mathbf{+0.98\%}$} & $\mathbf{-0.47\%}$ & \textcolor{gainred}{$\mathbf{+3.37\%}$} & $\mathbf{-0.00\%}$ & $\mathbf{-1.25\%}^{\dagger}$ \\
\addlinespace[2pt]
\midrule
\multicolumn{14}{l}{\textbf{\supconmix}\quad+\name{} : CUB~$-3$, Aircraft~$-1.5$, Cars~$-4$} \\
MSP & Baseline & $65.60\,{\pm}\,1.43$ & $59.80\,{\pm}\,1.19$ & $75.87\,{\pm}\,0.61$ & $71.62\,{\pm}\,1.01$ & $85.49\,{\pm}\,0.77$ & $77.91\,{\pm}\,0.75$ & $79.41\,{\pm}\,1.18$ & $70.59\,{\pm}\,0.80$ & $82.83\,{\pm}\,0.34$ & $77.70\,{\pm}\,1.22$ & $89.66\,{\pm}\,0.66$ & $80.89\,{\pm}\,0.65$ \\
 & +\name{} & $65.19\,{\pm}\,1.12$ & $59.86\,{\pm}\,1.29$ & $76.94\,{\pm}\,0.45$ & $72.09\,{\pm}\,0.65$ & $86.13\,{\pm}\,0.71$ & $78.17\,{\pm}\,0.92$ & $79.13\,{\pm}\,0.32$ & $70.92\,{\pm}\,0.62$ & $83.58\,{\pm}\,0.74$ & $77.72\,{\pm}\,0.98$ & $90.03\,{\pm}\,0.50$ & $80.93\,{\pm}\,0.80$ \\
 & gain & $\mathbf{-0.64\%}$ & \textcolor{gainred}{$\mathbf{+0.11\%}$} & \textcolor{gainred}{$\mathbf{+1.41\%}^{*}$} & \textcolor{gainred}{$\mathbf{+0.65\%}$} & \textcolor{gainred}{$\mathbf{+0.74\%}^{\dagger}$} & \textcolor{gainred}{$\mathbf{+0.34\%}$} & $\mathbf{-0.35\%}$ & \textcolor{gainred}{$\mathbf{+0.46\%}$} & \textcolor{gainred}{$\mathbf{+0.90\%}^{\dagger}$} & \textcolor{gainred}{$\mathbf{+0.03\%}$} & \textcolor{gainred}{$\mathbf{+0.41\%}$} & \textcolor{gainred}{$\mathbf{+0.05\%}$} \\
\addlinespace[2pt]
MLS & Baseline & $66.29\,{\pm}\,1.43$ & $58.80\,{\pm}\,0.96$ & $77.72\,{\pm}\,0.22$ & $71.01\,{\pm}\,1.24$ & $85.47\,{\pm}\,0.50$ & $77.72\,{\pm}\,0.72$ & $81.03\,{\pm}\,1.35$ & $69.85\,{\pm}\,0.59$ & $85.71\,{\pm}\,0.34$ & $77.56\,{\pm}\,1.50$ & $89.84\,{\pm}\,0.32$ & $80.88\,{\pm}\,0.62$ \\
 & +\name{} & $65.94\,{\pm}\,1.07$ & $58.91\,{\pm}\,0.94$ & $78.60\,{\pm}\,0.16$ & $71.48\,{\pm}\,0.77$ & $86.09\,{\pm}\,0.64$ & $78.01\,{\pm}\,1.01$ & $80.92\,{\pm}\,0.34$ & $70.29\,{\pm}\,0.30$ & $86.09\,{\pm}\,0.07$ & $77.52\,{\pm}\,1.12$ & $90.18\,{\pm}\,0.56$ & $80.93\,{\pm}\,0.87$ \\
 & gain & $\mathbf{-0.52\%}$ & \textcolor{gainred}{$\mathbf{+0.19\%}$} & \textcolor{gainred}{$\mathbf{+1.14\%}^{*}$} & \textcolor{gainred}{$\mathbf{+0.66\%}$} & \textcolor{gainred}{$\mathbf{+0.73\%}^{\dagger}$} & \textcolor{gainred}{$\mathbf{+0.37\%}$} & $\mathbf{-0.14\%}$ & \textcolor{gainred}{$\mathbf{+0.63\%}$} & \textcolor{gainred}{$\mathbf{+0.44\%}$} & $\mathbf{-0.05\%}$ & \textcolor{gainred}{$\mathbf{+0.38\%}$} & \textcolor{gainred}{$\mathbf{+0.06\%}$} \\
\addlinespace[2pt]
NNGuide & Baseline & $58.38\,{\pm}\,3.15$ & $51.32\,{\pm}\,2.30$ & $76.64\,{\pm}\,0.99$ & $70.65\,{\pm}\,1.50$ & $82.39\,{\pm}\,0.70$ & $76.27\,{\pm}\,0.34$ & $73.10\,{\pm}\,3.08$ & $63.02\,{\pm}\,1.82$ & $84.80\,{\pm}\,0.80$ & $77.31\,{\pm}\,1.77$ & $86.68\,{\pm}\,1.04$ & $79.68\,{\pm}\,0.36$ \\
 & +\name{} & $57.90\,{\pm}\,1.82$ & $51.19\,{\pm}\,1.58$ & $77.70\,{\pm}\,1.06$ & $71.01\,{\pm}\,0.60$ & $83.38\,{\pm}\,0.96$ & $76.84\,{\pm}\,0.70$ & $72.89\,{\pm}\,1.50$ & $63.30\,{\pm}\,0.78$ & $85.42\,{\pm}\,1.64$ & $77.19\,{\pm}\,0.92$ & $87.37\,{\pm}\,1.37$ & $79.96\,{\pm}\,0.51$ \\
 & gain & $\mathbf{-0.83\%}$ & $\mathbf{-0.25\%}$ & \textcolor{gainred}{$\mathbf{+1.38\%}^{*}$} & \textcolor{gainred}{$\mathbf{+0.50\%}$} & \textcolor{gainred}{$\mathbf{+1.20\%}$} & \textcolor{gainred}{$\mathbf{+0.74\%}$} & $\mathbf{-0.29\%}$ & \textcolor{gainred}{$\mathbf{+0.44\%}$} & \textcolor{gainred}{$\mathbf{+0.74\%}$} & $\mathbf{-0.15\%}$ & \textcolor{gainred}{$\mathbf{+0.80\%}$} & \textcolor{gainred}{$\mathbf{+0.35\%}$} \\
\addlinespace[2pt]
PostMax & Baseline & $66.58\,{\pm}\,1.43$ & $58.82\,{\pm}\,0.77$ & $78.66\,{\pm}\,0.13$ & $69.59\,{\pm}\,2.01$ & $85.59\,{\pm}\,0.18$ & $78.07\,{\pm}\,0.55$ & $81.37\,{\pm}\,1.37$ & $69.88\,{\pm}\,0.44$ & $86.94\,{\pm}\,0.70$ & $75.97\,{\pm}\,2.23$ & $89.96\,{\pm}\,0.08$ & $81.22\,{\pm}\,0.44$ \\
 & +\name{} & $66.25\,{\pm}\,1.14$ & $58.95\,{\pm}\,0.63$ & $79.31\,{\pm}\,0.74$ & $70.61\,{\pm}\,1.40$ & $86.20\,{\pm}\,0.40$ & $78.46\,{\pm}\,0.79$ & $81.25\,{\pm}\,0.46$ & $70.29\,{\pm}\,0.21$ & $86.99\,{\pm}\,0.84$ & $76.53\,{\pm}\,1.72$ & $90.28\,{\pm}\,0.63$ & $81.37\,{\pm}\,0.58$ \\
 & gain & $\mathbf{-0.49\%}$ & \textcolor{gainred}{$\mathbf{+0.22\%}$} & \textcolor{gainred}{$\mathbf{+0.82\%}$} & \textcolor{gainred}{$\mathbf{+1.46\%}$} & \textcolor{gainred}{$\mathbf{+0.71\%}^{\dagger}$} & \textcolor{gainred}{$\mathbf{+0.50\%}$} & $\mathbf{-0.15\%}$ & \textcolor{gainred}{$\mathbf{+0.59\%}$} & \textcolor{gainred}{$\mathbf{+0.05\%}$} & \textcolor{gainred}{$\mathbf{+0.74\%}$} & \textcolor{gainred}{$\mathbf{+0.36\%}$} & \textcolor{gainred}{$\mathbf{+0.18\%}$} \\
\addlinespace[2pt]
\midrule
\multicolumn{14}{l}{\textbf{\sphor}\quad+\name{}: CUB~$-8$, Aircraft~$-1.5$, Cars~$-0.5$} \\
MSP & Baseline & $63.99\,{\pm}\,1.35$ & $59.18\,{\pm}\,0.73$ & $69.10\,{\pm}\,1.94$ & $67.91\,{\pm}\,0.81$ & $84.27\,{\pm}\,0.22$ & $75.85\,{\pm}\,0.58$ & $77.96\,{\pm}\,1.08$ & $70.13\,{\pm}\,0.31$ & $75.50\,{\pm}\,1.98$ & $74.05\,{\pm}\,0.61$ & $88.20\,{\pm}\,0.20$ & $78.46\,{\pm}\,0.55$ \\
 & +\name{} & $65.24\,{\pm}\,0.98$ & $59.90\,{\pm}\,0.60$ & $69.36\,{\pm}\,1.25$ & $67.80\,{\pm}\,0.89$ & $84.47\,{\pm}\,0.18$ & $75.86\,{\pm}\,0.86$ & $79.25\,{\pm}\,0.57$ & $70.62\,{\pm}\,0.64$ & $75.76\,{\pm}\,1.55$ & $73.84\,{\pm}\,1.20$ & $88.35\,{\pm}\,0.21$ & $78.41\,{\pm}\,0.97$ \\
 & gain & \textcolor{gainred}{$\mathbf{+1.96\%}^{*}$} & \textcolor{gainred}{$\mathbf{+1.21\%}^{*}$} & \textcolor{gainred}{$\mathbf{+0.38\%}$} & $\mathbf{-0.16\%}$ & \textcolor{gainred}{$\mathbf{+0.23\%}$} & \textcolor{gainred}{$\mathbf{+0.01\%}$} & \textcolor{gainred}{$\mathbf{+1.65\%}^{\dagger}$} & \textcolor{gainred}{$\mathbf{+0.70\%}$} & \textcolor{gainred}{$\mathbf{+0.34\%}$} & $\mathbf{-0.28\%}$ & \textcolor{gainred}{$\mathbf{+0.16\%}$} & $\mathbf{-0.06\%}$ \\
\addlinespace[2pt]
MLS & Baseline & $64.70\,{\pm}\,1.24$ & $58.12\,{\pm}\,0.67$ & $70.28\,{\pm}\,1.82$ & $67.63\,{\pm}\,0.88$ & $84.31\,{\pm}\,0.24$ & $75.60\,{\pm}\,0.54$ & $79.88\,{\pm}\,0.92$ & $69.40\,{\pm}\,0.55$ & $77.42\,{\pm}\,1.83$ & $74.03\,{\pm}\,0.71$ & $88.47\,{\pm}\,0.14$ & $78.34\,{\pm}\,0.51$ \\
 & +\name{} & $65.72\,{\pm}\,0.97$ & $59.08\,{\pm}\,0.74$ & $70.50\,{\pm}\,1.26$ & $67.64\,{\pm}\,1.01$ & $84.55\,{\pm}\,0.25$ & $75.72\,{\pm}\,0.82$ & $80.49\,{\pm}\,0.57$ & $70.00\,{\pm}\,0.83$ & $77.64\,{\pm}\,1.55$ & $73.95\,{\pm}\,1.35$ & $88.59\,{\pm}\,0.21$ & $78.36\,{\pm}\,0.95$ \\
 & gain & \textcolor{gainred}{$\mathbf{+1.58\%}^{*}$} & \textcolor{gainred}{$\mathbf{+1.65\%}^{*}$} & \textcolor{gainred}{$\mathbf{+0.32\%}$} & \textcolor{gainred}{$\mathbf{+0.00\%}$} & \textcolor{gainred}{$\mathbf{+0.28\%}$} & \textcolor{gainred}{$\mathbf{+0.15\%}$} & \textcolor{gainred}{$\mathbf{+0.76\%}$} & \textcolor{gainred}{$\mathbf{+0.86\%}^{\dagger}$} & \textcolor{gainred}{$\mathbf{+0.28\%}$} & $\mathbf{-0.11\%}$ & \textcolor{gainred}{$\mathbf{+0.14\%}$} & \textcolor{gainred}{$\mathbf{+0.03\%}$} \\
\addlinespace[2pt]
NNGuide & Baseline & $61.16\,{\pm}\,1.31$ & $55.68\,{\pm}\,0.83$ & $66.05\,{\pm}\,2.32$ & $66.52\,{\pm}\,1.36$ & $83.20\,{\pm}\,0.23$ & $75.50\,{\pm}\,0.27$ & $75.83\,{\pm}\,1.40$ & $67.08\,{\pm}\,0.62$ & $71.81\,{\pm}\,2.47$ & $72.46\,{\pm}\,1.20$ & $87.82\,{\pm}\,0.30$ & $78.74\,{\pm}\,0.28$ \\
 & +\name{} & $61.91\,{\pm}\,0.98$ & $56.23\,{\pm}\,1.46$ & $66.32\,{\pm}\,2.01$ & $66.76\,{\pm}\,1.53$ & $83.30\,{\pm}\,0.18$ & $75.62\,{\pm}\,0.63$ & $76.03\,{\pm}\,0.93$ & $67.18\,{\pm}\,1.44$ & $72.10\,{\pm}\,2.45$ & $72.66\,{\pm}\,1.86$ & $87.69\,{\pm}\,0.42$ & $78.70\,{\pm}\,0.68$ \\
 & gain & \textcolor{gainred}{$\mathbf{+1.23\%}^{\dagger}$} & \textcolor{gainred}{$\mathbf{+1.00\%}$} & \textcolor{gainred}{$\mathbf{+0.41\%}$} & \textcolor{gainred}{$\mathbf{+0.36\%}$} & \textcolor{gainred}{$\mathbf{+0.12\%}$} & \textcolor{gainred}{$\mathbf{+0.16\%}$} & \textcolor{gainred}{$\mathbf{+0.26\%}$} & \textcolor{gainred}{$\mathbf{+0.15\%}$} & \textcolor{gainred}{$\mathbf{+0.40\%}$} & \textcolor{gainred}{$\mathbf{+0.29\%}$} & $\mathbf{-0.15\%}$ & $\mathbf{-0.06\%}$ \\
\addlinespace[2pt]
PostMax & Baseline & $65.02\,{\pm}\,1.13$ & $57.89\,{\pm}\,0.73$ & $73.11\,{\pm}\,1.39$ & $67.13\,{\pm}\,0.53$ & $84.26\,{\pm}\,0.47$ & $75.71\,{\pm}\,0.64$ & $80.58\,{\pm}\,0.76$ & $69.40\,{\pm}\,0.31$ & $80.98\,{\pm}\,1.39$ & $73.64\,{\pm}\,0.28$ & $88.44\,{\pm}\,0.15$ & $78.48\,{\pm}\,0.52$ \\
 & +\name{} & $65.91\,{\pm}\,0.86$ & $58.87\,{\pm}\,0.42$ & $73.14\,{\pm}\,1.40$ & $66.95\,{\pm}\,0.58$ & $84.45\,{\pm}\,0.49$ & $75.74\,{\pm}\,0.76$ & $80.88\,{\pm}\,0.47$ & $69.90\,{\pm}\,0.30$ & $80.95\,{\pm}\,1.63$ & $73.36\,{\pm}\,0.88$ & $88.52\,{\pm}\,0.27$ & $78.42\,{\pm}\,0.87$ \\
 & gain & \textcolor{gainred}{$\mathbf{+1.37\%}^{*}$} & \textcolor{gainred}{$\mathbf{+1.69\%}^{*}$} & \textcolor{gainred}{$\mathbf{+0.04\%}$} & $\mathbf{-0.27\%}^{*}$ & \textcolor{gainred}{$\mathbf{+0.22\%}$} & \textcolor{gainred}{$\mathbf{+0.04\%}$} & \textcolor{gainred}{$\mathbf{+0.37\%}$} & \textcolor{gainred}{$\mathbf{+0.72\%}^{\dagger}$} & $\mathbf{-0.04\%}$ & $\mathbf{-0.39\%}$ & \textcolor{gainred}{$\mathbf{+0.09\%}$} & $\mathbf{-0.08\%}$ \\
\addlinespace[2pt]
\midrule
\multicolumn{14}{l}{\textbf{\sphormix}\quad+\name{}: CUB~$-8$, Aircraft~$-2$, Cars~$-3$} \\
MSP & Baseline & $67.35\,{\pm}\,1.17$ & $61.70\,{\pm}\,0.58$ & $77.38\,{\pm}\,0.57$ & $72.12\,{\pm}\,0.30$ & $86.29\,{\pm}\,0.76$ & $78.67\,{\pm}\,1.26$ & $80.76\,{\pm}\,0.92$ & $72.24\,{\pm}\,0.24$ & $84.43\,{\pm}\,0.50$ & $78.02\,{\pm}\,0.36$ & $90.25\,{\pm}\,0.43$ & $81.39\,{\pm}\,1.03$ \\
 & +\name{} & $68.26\,{\pm}\,0.93$ & $62.20\,{\pm}\,0.50$ & $77.55\,{\pm}\,0.32$ & $72.04\,{\pm}\,0.42$ & $87.13\,{\pm}\,0.39$ & $79.05\,{\pm}\,0.53$ & $81.83\,{\pm}\,0.57$ & $72.62\,{\pm}\,0.36$ & $84.62\,{\pm}\,0.40$ & $77.90\,{\pm}\,0.50$ & $90.89\,{\pm}\,0.05$ & $81.57\,{\pm}\,0.36$ \\
 & gain & \textcolor{gainred}{$\mathbf{+1.35\%}^{*}$} & \textcolor{gainred}{$\mathbf{+0.81\%}^{*}$} & \textcolor{gainred}{$\mathbf{+0.22\%}$} & $\mathbf{-0.12\%}$ & \textcolor{gainred}{$\mathbf{+0.98\%}^{\dagger}$} & \textcolor{gainred}{$\mathbf{+0.49\%}$} & \textcolor{gainred}{$\mathbf{+1.32\%}^{\dagger}$} & \textcolor{gainred}{$\mathbf{+0.53\%}^{*}$} & \textcolor{gainred}{$\mathbf{+0.23\%}$} & $\mathbf{-0.15\%}$ & \textcolor{gainred}{$\mathbf{+0.71\%}^{\dagger}$} & \textcolor{gainred}{$\mathbf{+0.23\%}$} \\
\addlinespace[2pt]
MLS & Baseline & $67.90\,{\pm}\,1.12$ & $60.26\,{\pm}\,0.59$ & $78.34\,{\pm}\,0.53$ & $70.97\,{\pm}\,0.22$ & $86.20\,{\pm}\,0.84$ & $78.28\,{\pm}\,1.21$ & $82.18\,{\pm}\,0.88$ & $70.90\,{\pm}\,0.41$ & $85.87\,{\pm}\,0.51$ & $76.94\,{\pm}\,0.19$ & $90.36\,{\pm}\,0.55$ & $81.17\,{\pm}\,0.94$ \\
 & +\name{} & $68.67\,{\pm}\,0.92$ & $60.98\,{\pm}\,0.51$ & $78.54\,{\pm}\,0.23$ & $71.12\,{\pm}\,0.75$ & $87.04\,{\pm}\,0.46$ & $78.66\,{\pm}\,0.55$ & $82.96\,{\pm}\,0.59$ & $71.52\,{\pm}\,0.50$ & $86.08\,{\pm}\,0.26$ & $77.11\,{\pm}\,0.85$ & $90.95\,{\pm}\,0.16$ & $81.31\,{\pm}\,0.33$ \\
 & gain & \textcolor{gainred}{$\mathbf{+1.12\%}^{\dagger}$} & \textcolor{gainred}{$\mathbf{+1.19\%}^{*}$} & \textcolor{gainred}{$\mathbf{+0.26\%}$} & \textcolor{gainred}{$\mathbf{+0.22\%}$} & \textcolor{gainred}{$\mathbf{+0.97\%}^{\dagger}$} & \textcolor{gainred}{$\mathbf{+0.48\%}$} & \textcolor{gainred}{$\mathbf{+0.95\%}$} & \textcolor{gainred}{$\mathbf{+0.88\%}^{*}$} & \textcolor{gainred}{$\mathbf{+0.25\%}$} & \textcolor{gainred}{$\mathbf{+0.22\%}$} & \textcolor{gainred}{$\mathbf{+0.65\%}$} & \textcolor{gainred}{$\mathbf{+0.16\%}$} \\
\addlinespace[2pt]
NNGuide & Baseline & $63.82\,{\pm}\,0.85$ & $55.98\,{\pm}\,0.95$ & $76.20\,{\pm}\,0.23$ & $70.37\,{\pm}\,0.36$ & $84.05\,{\pm}\,1.42$ & $76.94\,{\pm}\,0.82$ & $77.80\,{\pm}\,0.55$ & $66.54\,{\pm}\,1.01$ & $82.99\,{\pm}\,0.40$ & $76.00\,{\pm}\,0.48$ & $88.17\,{\pm}\,1.28$ & $79.99\,{\pm}\,0.59$ \\
 & +\name{} & $64.50\,{\pm}\,0.80$ & $57.31\,{\pm}\,1.10$ & $76.33\,{\pm}\,0.51$ & $70.72\,{\pm}\,1.01$ & $85.05\,{\pm}\,0.65$ & $77.60\,{\pm}\,0.55$ & $77.99\,{\pm}\,0.27$ & $67.59\,{\pm}\,0.96$ & $83.11\,{\pm}\,0.34$ & $76.34\,{\pm}\,1.21$ & $88.88\,{\pm}\,0.47$ & $80.36\,{\pm}\,0.51$ \\
 & gain & \textcolor{gainred}{$\mathbf{+1.07\%}$} & \textcolor{gainred}{$\mathbf{+2.38\%}^{*}$} & \textcolor{gainred}{$\mathbf{+0.17\%}$} & \textcolor{gainred}{$\mathbf{+0.49\%}$} & \textcolor{gainred}{$\mathbf{+1.19\%}$} & \textcolor{gainred}{$\mathbf{+0.85\%}$} & \textcolor{gainred}{$\mathbf{+0.24\%}$} & \textcolor{gainred}{$\mathbf{+1.57\%}^{*}$} & \textcolor{gainred}{$\mathbf{+0.15\%}$} & \textcolor{gainred}{$\mathbf{+0.46\%}$} & \textcolor{gainred}{$\mathbf{+0.81\%}$} & \textcolor{gainred}{$\mathbf{+0.47\%}$} \\
\addlinespace[2pt]
PostMax & Baseline & $68.19\,{\pm}\,1.13$ & $60.41\,{\pm}\,0.73$ & $78.78\,{\pm}\,0.46$ & $69.88\,{\pm}\,0.41$ & $86.15\,{\pm}\,0.75$ & $78.22\,{\pm}\,0.70$ & $82.52\,{\pm}\,0.91$ & $71.03\,{\pm}\,0.48$ & $86.34\,{\pm}\,0.43$ & $75.65\,{\pm}\,0.43$ & $90.32\,{\pm}\,0.48$ & $81.12\,{\pm}\,0.43$ \\
 & +\name{} & $68.87\,{\pm}\,1.02$ & $61.22\,{\pm}\,0.66$ & $79.14\,{\pm}\,0.14$ & $70.09\,{\pm}\,0.80$ & $86.90\,{\pm}\,0.43$ & $78.55\,{\pm}\,0.31$ & $83.19\,{\pm}\,0.75$ & $71.73\,{\pm}\,0.53$ & $86.73\,{\pm}\,0.31$ & $75.88\,{\pm}\,0.80$ & $90.82\,{\pm}\,0.17$ & $81.21\,{\pm}\,0.20$ \\
 & gain & \textcolor{gainred}{$\mathbf{+1.01\%}$} & \textcolor{gainred}{$\mathbf{+1.33\%}^{*}$} & \textcolor{gainred}{$\mathbf{+0.46\%}$} & \textcolor{gainred}{$\mathbf{+0.31\%}$} & \textcolor{gainred}{$\mathbf{+0.87\%}^{\dagger}$} & \textcolor{gainred}{$\mathbf{+0.42\%}$} & \textcolor{gainred}{$\mathbf{+0.81\%}$} & \textcolor{gainred}{$\mathbf{+0.99\%}^{*}$} & \textcolor{gainred}{$\mathbf{+0.46\%}^{*}$} & \textcolor{gainred}{$\mathbf{+0.31\%}$} & \textcolor{gainred}{$\mathbf{+0.55\%}$} & \textcolor{gainred}{$\mathbf{+0.12\%}$} \\
\addlinespace[2pt]
\bottomrule
\end{tabular}

}
\end{table*}

\begin{table*}[!t]
\centering
\scriptsize
\caption{\textbf{Fine-grained OSR performance on the Semantic Shift Benchmark (SCARS, FGVC-Aircraft, and CUB): Closed-set Accuracy}  Gains are in blue. The results are given as the mean $\pm$ sd over three training runs with fixed seeds.} 
\label{tab:osr_main_accuracy}
\resizebox{\textwidth}{!}{
\begin{tabular}{l rrr rrr rrr}
\toprule
& \multicolumn{3}{c}{CUB} & \multicolumn{3}{c}{Aircraft} & \multicolumn{3}{c}{Cars} \\
\cmidrule(lr){2-4}\cmidrule(lr){5-7}\cmidrule(lr){8-10}
Method &  & +\name{} & $\Delta\%$
& & +\name{} & $\Delta\%$
&  & +\name{} & $\Delta\%$ \\
\midrule
%CE & $70.53\,{\pm}\,0.86$ & --- & --- & $82.51\,{\pm}\,0.24$ & --- & --- & $91.38\,{\pm}\,0.23$ & --- & --- \\
%\supcon{} (w/o LS) & $71.16\,{\pm}\,1.23$ & $72.15\,{\pm}\,0.29$ & \textcolor{gainred}{$\mathbf{+1.38\%}$} & $83.11\,{\pm}\,0.43$ & $82.65\,{\pm}\,0.48$ & $\mathbf{-0.55\%}^{*}$ & $91.12\,{\pm}\,0.58$ & $91.41\,{\pm}\,0.34$ & \textcolor{gainred}{$\mathbf{+0.32\%}$} \\

\supcon{} & $71.74\,{\pm}\,1.16$ & $72.29\,{\pm}\,0.62$ & \textcolor{gainred}{$\mathbf{+0.77\%}$} & $81.21\,{\pm}\,1.62$ & $82.99\,{\pm}\,0.52$ & \textcolor{gainred}{$\mathbf{+2.19\%}$} & $90.87\,{\pm}\,0.87$ & $91.40\,{\pm}\,0.32$ & \textcolor{gainred}{$\mathbf{+0.58\%}$} \\
\supconmix{}  & $75.89\,{\pm}\,0.53$ & $75.42\,{\pm}\,1.23$ & $\mathbf{-0.62\%}$ & $85.73\,{\pm}\,0.67$ & $86.19\,{\pm}\,0.73$ & \textcolor{gainred}{$\mathbf{+0.54\%}$} & $92.67\,{\pm}\,0.18$ & $93.26\,{\pm}\,0.53$ & \textcolor{gainred}{$\mathbf{+0.64\%}$} \\
\sphor{} & $74.20\,{\pm}\,0.88$ & $74.58\,{\pm}\,0.75$ & \textcolor{gainred}{$\mathbf{+0.51\%}^{\dagger}$} & $83.53\,{\pm}\,0.76$ & $83.67\,{\pm}\,0.66$ & \textcolor{gainred}{$\mathbf{+0.17\%}$} & $91.80\,{\pm}\,0.62$ & $91.86\,{\pm}\,0.45$ & \textcolor{gainred}{$\mathbf{+0.07\%}$} \\
\sphormix & $77.05\,{\pm}\,0.84$ & $77.29\,{\pm}\,0.94$ & \textcolor{gainred}{$\mathbf{+0.31\%}$} & $85.83\,{\pm}\,0.46$ & $85.91\,{\pm}\,0.24$ & \textcolor{gainred}{$\mathbf{+0.09\%}$} & $92.93\,{\pm}\,0.69$ & $93.32\,{\pm}\,0.51$ & \textcolor{gainred}{$\mathbf{+0.42\%}^{\dagger}$} \\
\bottomrule
\end{tabular}}
\end{table*}

\section{Results and Discussion} \label{resultssection}

\noindent \textbf{Superiority of Augmented Variants in SSB Benchmark:}
Table~\ref{tab:osr_main} reports OSR results on the SSB Benchmark. Our proposed \name{} boosts performance in open-set recognition along both axis: it improves separability from unknowns as shown via AUROC, and it also improves closed-set performance (As shown in Table~\ref{tab:osr_main_accuracy}). In summation, this net result reflects in improved OSCR in Table~\ref{tab:osr_main} across all training-testing regimes.  While minor fluctuations occur in specific sub-metrics, likely due to the absence of dataset-specific hyperparameter tuning, the overall consistency across three distinct datasets underscores the generalizability of our approach. We further show it is consistent across the Operating Curves in Appendix~\ref{app:operatingcurve} and that it comes at an insignificant performance tax quantified via FLOPS in Appendix~\ref{app:flopsanalysis}.

\textbf{Empirical Validation:} While \name{} intervenes via the loss in the projection space, the inference happens in the encoder representations. Thus, to probe the geometric properties of the representations induced by \name{} and relate them to OSR, we use two metrics defined in \citep{li2022targeted}.
\begin{definition}[Intra-class compactness (Comp), Inter-class dispersion (Disp)]
\label{def:comp-disp}
Let $\mathcal{F}_i=\{\mathrm{Enc}(\mathbf{x}) :
(\mathbf{x},y)\in\mathcal{D}_K,\ y=i\}$ and
$\bar{\mathbf{v}}_i = \sum_{\mathbf{v}\in\mathcal{F}_i}\mathbf{v}\,/\,
\lVert\sum_{\mathbf{v}\in\mathcal{F}_i}\mathbf{v}\rVert_2$. Then:
\begin{align}
\mathrm{Comp} &\triangleq \frac{1}{C}\sum_{i=1}^{C}
  \frac{1}{\lvert\mathcal{F}_i\rvert^{2}}
  \sum_{\mathbf{v}_j,\mathbf{v}_k\in\mathcal{F}_i}
  \lVert \mathbf{v}_j-\mathbf{v}_k \rVert_2 \quad(\downarrow), \quad \mathrm{Disp} &\triangleq \frac{1}{C(C-1)}\sum_{i\neq j}
  \lVert \bar{\mathbf{v}}_i-\bar{\mathbf{v}}_j \rVert_2 \quad(\uparrow).
\label{eq:disp}
\end{align}
\end{definition}

We observe that, in Table~\ref{tab:geom_full}, it produces compact, well-separated geometry that Section~\ref{sec:method} links to reduced
open-space risk. \name{} reduces Comp in 11 of 12 settings and increases Disp in all 12, so the tighter classes do not come at the cost of inter-class
separation. The geometry is also consistent with open-set performance. The single setting where Comp does not improve, \supconmix{} on CUB, also regresses in the easy split under every post-processor (Table~\ref{tab:osr_main}). Conversely, \sphormix{},
which shows the largest and most consistently significant geometric
gains, also shows the most consistent open-set improvements. The feature geometry is analyzed further using the lens of Angular Separability and Norm Separability in Appendix~\ref{app:asandns}~\citep{sphor}. Extended Ablations are presented in Appendix~\ref{app:extended_ablations}.

\begin{table}[t]
\centering
\scriptsize
\caption{\textbf{Feature geometry in terms of the gains in Intra-class compactness (Comp) and Inter-class dispersion (Disp).}
Each $\Delta$ reflects the gain Comp/Disp when the baseline method was augmented with \name{}. Colored in \gaincolor{} means the intervention improved the feature geometry in that aspect. Results over 3 seeds. }
{\footnotesize
\setlength{\tabcolsep}{3pt}\renewcommand{\arraystretch}{1.0}
\begin{tabular}{@{}l rr rr rr rr@{}}
\toprule
 & \multicolumn{2}{c}{\supcon{}} & \multicolumn{2}{c}{\supconmix{}}
 & \multicolumn{2}{c}{\sphor{}} & \multicolumn{2}{c}{\sphormix{}} \\
\cmidrule(lr){2-3}\cmidrule(lr){4-5}\cmidrule(lr){6-7}\cmidrule(lr){8-9}
 & Comp.\,$\downarrow$ & Disp.\,$\uparrow$
 & Comp.\,$\downarrow$ & Disp.\,$\uparrow$
 & Comp.\,$\downarrow$ & Disp.\,$\uparrow$
 & Comp.\,$\downarrow$ & Disp.\,$\uparrow$ \\
\midrule
CUB      & \textcolor{gainred}{$-0.0052$} & \textcolor{gainred}{$+0.0111$}
         & $+0.0005$                      & \textcolor{gainred}{$+0.0035$}
         & \textcolor{gainred}{$-0.0135$} & \textcolor{gainred}{$+0.0004$}
         & \textcolor{gainred}{$-0.0152$} & \textcolor{gainred}{$+0.0256$} \\
Aircraft & \textcolor{gainred}{$-0.0167$} & \textcolor{gainred}{$+0.0166$}
         & \textcolor{gainred}{$-0.0088$} & \textcolor{gainred}{$+0.0169$}
         & \textcolor{gainred}{$-0.0059$} & \textcolor{gainred}{$+0.0126$}
         & \textcolor{gainred}{$-0.0148$} & \textcolor{gainred}{$+0.0360$} \\
Cars     & \textcolor{gainred}{$-0.0124$} & \textcolor{gainred}{$+0.0155$}
         & \textcolor{gainred}{$-0.0068$} & \textcolor{gainred}{$+0.0250$}
         & \textcolor{gainred}{$-0.0052$} & \textcolor{gainred}{$+0.0058$}
         & \textcolor{gainred}{$-0.0201$} & \textcolor{gainred}{$+0.0457$} \\
\bottomrule
\end{tabular}}

\label{tab:geom_full}
\end{table}

\section{Limitations and Future Work}
\label{conlusion}

\name{} redistributes label smoothing around a base coefficient $\epsilon$ rather
than adding smoothing of its own, so its effect scales with how much
smoothing the base objective uses. objectives tuned to little or no
smoothing leave it little to redistribute. Thus, the next natural step is to make the base coefficient $\epsilon$ adaptive.

\subsection*{AI use statement}
Generative AI tools were only used in the preparation of this manuscript for purposes of grammar checking. No AI systems were used for data analysis, model development, experimental design, or result interpretation. All research decisions and scientific contributions remain the responsibility of the authors.

\subsection*{Ethics statement}
The research conforms with the ICLR Code of Ethics. This work involves standard image classification benchmarks with no personal data, sensitive content, or harmful applications.

\subsection*{Reproducibility statement}
We use only the public Semantic Shift Benchmark with its published splits. All hyperparameters, backbone initialisation and hardware are listed in Appendix~\ref{app:details}, and every number is the mean over three fixed seeds (1, 2, 3). Code and configuration files are also included in the supplementary materiels.

\bibliography{egbib}
\bibliographystyle{iclr2027_conference}

\newpage
\appendix

\section{Evaluation Metrics}
\label{app:metrics}

To begin, we define a supervised training set for a classification task as 
$\mathcal{D}_{\text{train}} = \{(\mathbf{x}_i, y_i)\}_{i=1}^{N} \subset \mathcal{X} \times \mathcal{C}$, where $\mathcal{X}$ is the image space and $\mathcal{C}$ is the set of `known' classes.

In the closed-set scenario, the model is evaluated on a test set where the labels are strictly drawn from the training categories:
$\mathcal{D}_{\text{M}} := \mathcal{D}_{k} = \{(\mathbf{x}_j, y_j)\}_{j=1}^{M} \subset \mathcal{X} \times \mathcal{C}$. In this setting, the model returns a distribution over known classes as $p(y \mid \mathbf{x})$.

Conversely, in Open-Set Recognition (OSR), test images may also originate from unseen classes $\mathcal{U}$, where $\mathcal{U}$ denotes the set of unknown classes and $\mathcal{C} \cap \mathcal{U} = \emptyset$.  The OSR-evaluation test set is a union of known and unknown samples: $\mathcal{D}_{\text{Open}} = \mathcal{D}_{k} \cup \mathcal{D}_{u}$. 
The unknown component is defined as:
$\mathcal{D}_{u} = \{(\mathbf{x}_j, y_j)\}_{j=1}^{K} \subset \mathcal{X} \times \mathcal{U}$.

In the open-set setting, the model provides both a class distribution $p(y \mid \mathbf{x}, y \in \mathcal{C})$ and a score function $\mathcal{S}(\mathbf{x})$ to indicate whether or not the test sample belongs to any of the known classes. The final prediction $\hat{y}$ is determined by an operational threshold $\theta$:

\begin{equation}
\hat{y} = 
\begin{cases} 
\operatorname*{argmax}_{y \in \mathcal{C}} \, p(y \mid \mathbf{x}) & \text{if } \mathcal{S}(\mathbf{x}) \geq \theta \\ 
\text{Unknown} & \text{if } \mathcal{S}(\mathbf{x}) < \theta 
\end{cases}
\end{equation}

We evaluate models with three metrics: 

Standard \textbf{Closed-set Accuracy} is:
\begin{equation}
\text{Acc} = \frac{1}{|\mathcal{D}_{k}|} \sum_{i=1}^{|\mathcal{D}_{k}|} \mathbb{I}(\hat{y}_i = y_i)
\end{equation}

The \textbf{Area Under the Receiver Operating Characteristic (AUROC)} is a threshold-independent metric utilized to evaluate the model's capability to discriminate between known classes $\mathcal{C}$ and unknown classes $\mathcal{U}$ based on the scoring function $\mathcal{S}(\mathbf{x})$. However, this metric is flawed as it doesn't account for the classification accuracy.  Thus, we need to establish a link between AUROC (unknown rejection) and known-class classification across varying thresholds. For a given operational threshold $\theta$, we define the \textbf{Correct Classification Rate (CCR)\citep{OSCR}} and the \textbf{False Positive Rate (FPR)\citep{OSCR}} as follows:
\begin{equation}
\text{CCR}(\theta) = \frac{1}{|\mathcal{D}_{k}|} \sum_{\mathbf{x}_i \in \mathcal{D}_{k}} \mathbb{I}(y_{i} = \hat{y}_i ) \cdot \mathbb{I}(\mathcal{S}(\mathbf{x}_i) > \theta ), \quad \text{FPR}(\theta) = \frac{1}{|\mathcal{D}_{u}|} \sum_{\mathbf{x}_i \in \mathcal{D}_{u}} \mathbb{I}(\mathcal{S}(\mathbf{x}_i) > \theta )
\end{equation}

Thus, CCR contextualizes Acc for OSR, taking into account if the known class is mistakenly rejected as unknown or if the unknown class is mistakenly classified as known. However, it requires a threshold $\theta$. Based on CCR and FPR, we define two metrics to measure OSR performance: a threshold-independent metric, the Open-set Classification Rate (OSCR), and a threshold-dependent CCR5 (where $FPR=5\%$). 

The \textbf{Open-set Classification Rate (OSCR)\citep{OSCR}} summarizes the performance of a model across all possible operational thresholds $\theta$. It is defined as the area under the curve formed by plotting the Correct Classification Rate (CCR) against the False Positive Rate (FPR):
\begin{equation}
\text{OSCR} = \int_{0}^{1} \text{CCR}(\text{FPR}) d\text{FPR}
\end{equation}

For many safety-critical applications, this is a standard threshold. While recent OSR literature has introduced metrics such as Open-Set Accuracy (OOSA), we do not employ it in our evaluation. OOSA requires a surrogate set, often from a different dataset, to estimate the operational threshold $\theta$. We argue that this approach does not accurately measure the \textit{Semantic Shift}, as the threshold estimation is decoupled from the actual semantic boundary between known and unknown classes in the target domain, and it requires a validation dataset.

\section{Experimental Details}
\label{app:details}

\subsection{Hyperparameter and Data Augmentation}
\label{app:Hyperparameterdetails}

We experimented with different loss functions and hyperparameters tailored to each dataset. A crucial component of our training was the use of \textbf{RandAugment}, a data augmentation technique that automatically finds the optimal combination of augmentation policies. The specific hyperparameters used for each loss function and dataset combination are detailed in Table \ref{tab:hyperparameters}. The RandAugment policy is controlled by two main parameters: $N$, the number of augmentation transformations to apply, and $M$, the magnitude of the transformations.

\begin{table*}[h!]
\centering
\scriptsize
\caption{Hyperparameters for Fine-Grained Datasets. We take the exact parameters from MLS\citep{vaze2022openset}}.
\label{tab:hyperparameters}
\begin{tabular}{|c|c|c|c|c|c|c|}
\hline
\textbf{Loss} & \textbf{Dataset} & \textbf{Image Size} & \textbf{RandAug N} & \textbf{RandAug M} & \textbf{Label Smoothing} & \textbf{Batch Size} \\
\hline
SpHOR/Sphor-lite & CUB & 224 & 2 & 30 & 0.3 & 128 \\
SpHOR/Sphor-lite & FGV-Aircraft & 224 & 2 & 15 & 0.2 & 128 \\
SpHOR/Sphor-lite & SCARS & 224 & 2 & 15 & 0.2 & 128 \\
ConOSR/SupCON & CUB & 224 & 2 & 30 & 0.3 & 128 \\
ConOSR/SupCON & FGV-Aircraft & 224 & 2 & 15 & 0.2 & 128 \\
ConOSR/SupCON & SCARS & 224 & 2 & 15 & 0.2 & 128 \\
\hline
\end{tabular}
\end{table*}

\subsection{Stage 02 Training}
\label{app:stage02} 
We freeze the stage-1 backbone and retrain only the linear
classifier $\mathrm{Linear}(2048{\to}C)$ on its features, starting from the stage-1
head. We pass each training image through the frozen network once, with the same
RandAugment setting as stage~1, and cache the 2048-d features. The head then
trains on this cache for 20 epochs with batch size 128, using standard
cross-entropy and no label smoothing. We optimise with Adam at learning rate
$10^{-4}$, with no weight decay and no learning-rate schedule. The head is
retrained once per stage-1 checkpoint, so each seed gets its own head, and every
open-set score is computed from that head.

\subsection{Dataset}
\label{app:dataset_Details}

The fundamental challenge in OSR is distinguishing between Covariate Shift and Semantic Shift. While covariate shift refers to changes in the marginal input distribution $P_{\text{test}}(\mathbf{x}) \neq P_{\text{train}}(\mathbf{x})$ without changing label semantics $P_{\text{test}}(y \mid \mathbf{x}) = P_{\text{train}}(y \mid \mathbf{x})$, OSR induces a semantic shift. Here, the conditional label distribution changes due to novel classes $P_{\text{test}}(y \mid \mathbf{x}) \neq P_{\text{train}}(y \mid \mathbf{x})$, while the input domain ideally remains static $P_{\text{test}}(\mathbf{x}) \approx P_{\text{train}}(\mathbf{x})$.

We conducted experiments on the Semantic Shift Benchmark (SSB)~\citep{vaze2022openset}, which uses three fine-grained classification datasets: CUB~\citep{Wah2011CaltechUCSD}, FGVC-Aircraft~\citep{Krause20133D}, and Stanford Cars~\citep{Maji2013FineGrained}.
For each dataset, SSB partitions the original categories into 
\emph{known} classes used for training and 
\emph{unknown} classes reserved for evaluation.

\textbf{CUB-200-2011.} 
CUB-200-2011 is a fine-grained bird classification dataset 
containing 11,788 images from 200 bird species. 
In the SSB protocol, 100 species are treated as known classes, while the remaining species are divided into 
multiple groups of unknown classes according to semantic similarity. 
The dataset is challenging due to subtle inter-class differences 
and limited training samples per class.

\textbf{Stanford Cars.} 
The Stanford Cars dataset contains 16,185 images 
across 196 car model categories. 
Following the SSB setup, 98 classes are used as known categories, and the remaining classes serve as unknown categories 
under different semantic shift levels. 
The dataset includes large viewpoint variations and fine-grained 
distinctions at the level of production year and model.

\textbf{FGVC-Aircraft.} 
FGVC-Aircraft consists of 10,000 images over 100 aircraft variants. 
The dataset provides a semantic hierarchy from manufacturer to model.  50 aircraft categories are used as known classes, 
while the remaining categories are treated as unknown classes. 
Semantic partitions are constructed based on hierarchical similarity.

\begin{table}[h!]
\centering
\caption{Statistics of datasets used in the Semantic Shift Benchmark (SSB). 
Known classes are used for training, while unknown classes are reserved for open-set evaluation.}
\begin{tabular}{lcccc}
\hline
Dataset & Total Images & Total Classes & Known Classes & Unknown Classes \\
\hline
CUB & 11,788 & 200 & 100 & 100 \\
Stanford Cars & 16,185 & 196 & 98 & 98 \\
FGVC-Aircraft & 10,000 & 100 & 50 & 50 \\
\hline
\end{tabular}
\label{tab:ssb_stats}
\end{table}

\subsection{Scoring Functions} \label{appendix:scoringfunction}
Let $f(\mathbf{x}) \in \mathbb{R}^{C}$ be the classifier logits over $C$ classes, and let $\mathbf{z} = \phi(\mathbf{x})$ be the feature vector taken from the layer just before the classifier.

\paragraph{MSP.} The largest softmax probability over all classes:
\begin{equation}
S_{\mathrm{MSP}}(\mathbf{x}) = \max_{c} \frac{\exp f_c(\mathbf{x})}{\sum_{j} \exp f_j(\mathbf{x})}.
\end{equation}

\paragraph{MLS.} The largest logit over all classes:
\begin{equation}
S_{\mathrm{MLS}}(\mathbf{x}) = \max_{c} f_c(\mathbf{x}).
\end{equation}

\paragraph{PostMax.} PostMax normalizes the maximum logit by the feature magnitude and maps it through the cumulative distribution of a Generalized Pareto Distribution (GPD) fit on correctly classified training samples:
\begin{equation}
S_{\mathrm{PostMax}}(\mathbf{x}) = \Psi_{\mu, \sigma, \xi}\!\left( \frac{\max_c f_c(\mathbf{x})}{\|\mathbf{z}\|} \right),
\end{equation}
where $\Psi_{\tilde{\mu}, \tilde{\sigma}, \tilde{\xi}}$ is the GPD CDF with location $\tilde{\mu}$, scale $\tilde{\sigma}$, and shape $\tilde{\xi}$ estimated by maximum likelihood on the normalized max-logits of correctly classified training samples. Dividing by $\|\mathbf{z}\|$ counteracts the observation that unknown samples on large-scale models tend to have larger feature norms, and the GPD turns the score into a calibrated probability.

% \paragraph{KNN.} Given a bank of ID features $\{\mathbf{z}_1, \dots, \mathbf{z}_n\}$, sort them by cosine similarity to the test feature $\mathbf{z}$ in descending order, and take the similarity to the $k$-th neighbor:
% \begin{equation}
% S_{\mathrm{KNN}}(\mathbf{x}) = \mathrm{sim}(\mathbf{z}_{(k)}, \mathbf{z}).
% \end{equation}

\paragraph{NNGuide.} Multiply a base classifier score $S_{\mathrm{base}}$ (we use negative energy, $S_{\mathrm{base}}(\mathbf{x}) = \log \sum_{c} \exp f_c(\mathbf{x})$) by the average of the top-$k$ confidence-weighted similarities:
\begin{equation}
S_{\mathrm{NNGuide}}(\mathbf{x}) = S_{\mathrm{base}}(\mathbf{x}) \cdot \frac{1}{k} \sum_{i=1}^{k} s_{(i)} \, \mathrm{sim}(\mathbf{z}_{(i)}, \mathbf{z}),
\end{equation}
where $s_i = S_{\mathrm{base}}(\mathbf{x}_i)$ is the base score of bank sample $i$, and the indices $(i)$ are sorted by descending $s_i \cdot \mathrm{sim}(\mathbf{z}_i, \mathbf{z})$. Weighting neighbors by their confidence pulls the search toward high-confidence ID regions, which lowers scores on far-OOD inputs while still separating near-OOD inputs around class boundaries.

\section{Extended Results}

\subsection{Operating Curve Analysis}
\label{app:operatingcurve}

\providecommand{\oscrpanelw}{0.47\textwidth}

\begin{figure}[p]
\centering

\begin{minipage}[t]{\oscrpanelw}\centering
{\scriptsize\textbf{\supcon{}} --- Easy Split}\par\vspace{1pt}
\includegraphics[width=\textwidth]{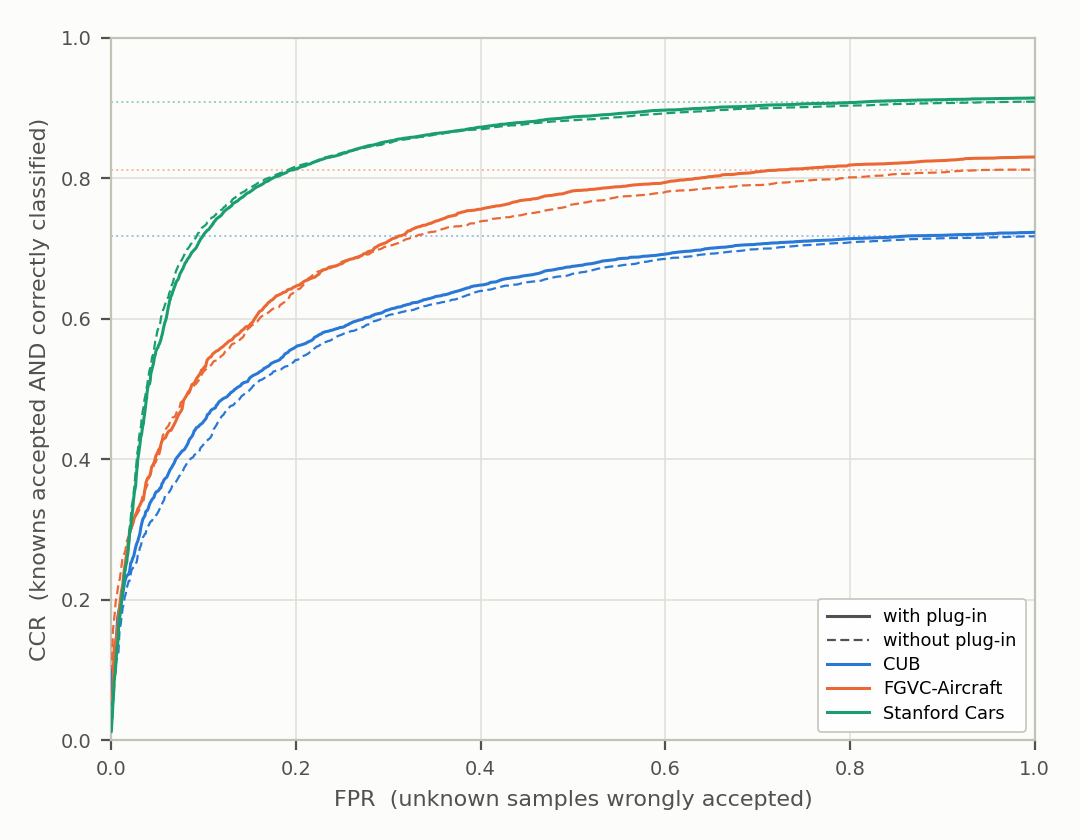}
\end{minipage}\hfill
\begin{minipage}[t]{\oscrpanelw}\centering
{\scriptsize\textbf{\supcon{}} --- Hard Split}\par\vspace{1pt}
\includegraphics[width=\textwidth]{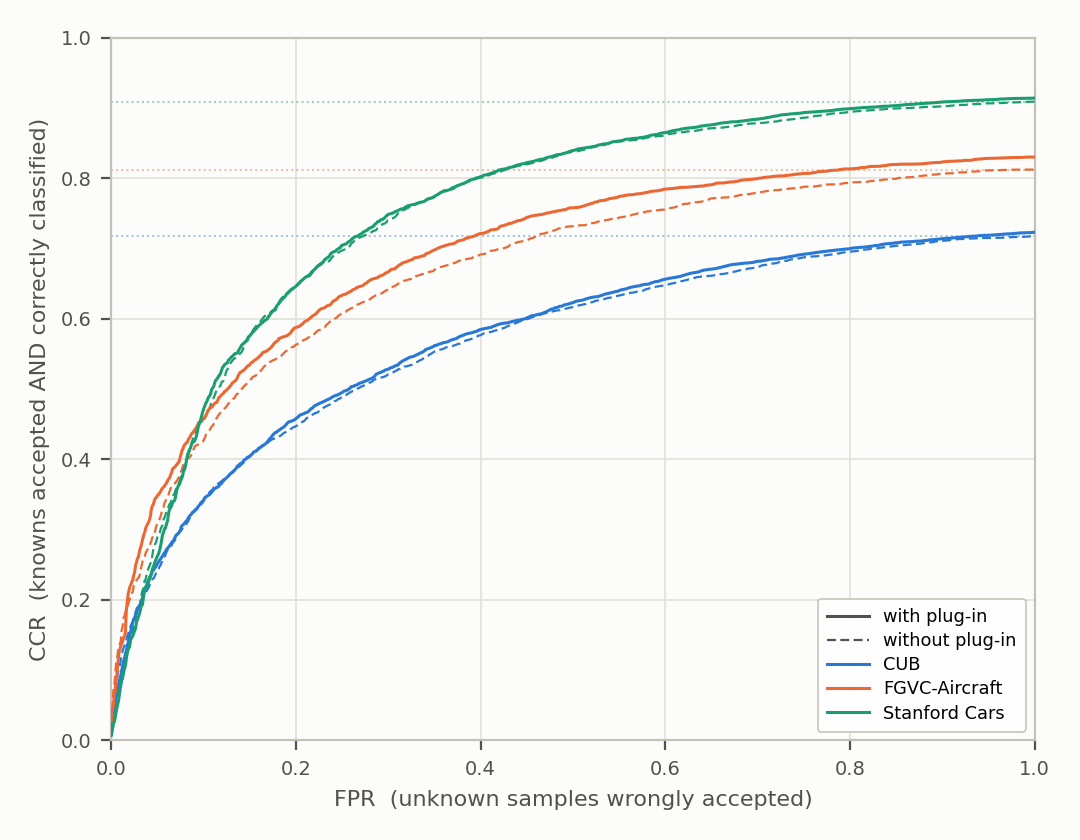}
\end{minipage}
\par\smallskip

\caption{\textbf{OSCR operating curves for \supcon{} against SSB with the MSP post-processor.}   \textbf{Solid} indicates with \name{} ; \textbf{dashed} is the baseline without the plug-in. The curves are vertically averaged over 3 seeds for a fixed FPR.  The results are shown for both the Easy and Hard Splits.  \name{} pushes out baselines almost consistently across all operating points.}
\label{fig:oscr_grid_a1}
\end{figure}

\begin{figure}[p]
\centering

\begin{minipage}[t]{\oscrpanelw}\centering
{\scriptsize\textbf{\supconmix{}} --- Easy Split}\par\vspace{1pt}
\includegraphics[width=\textwidth]{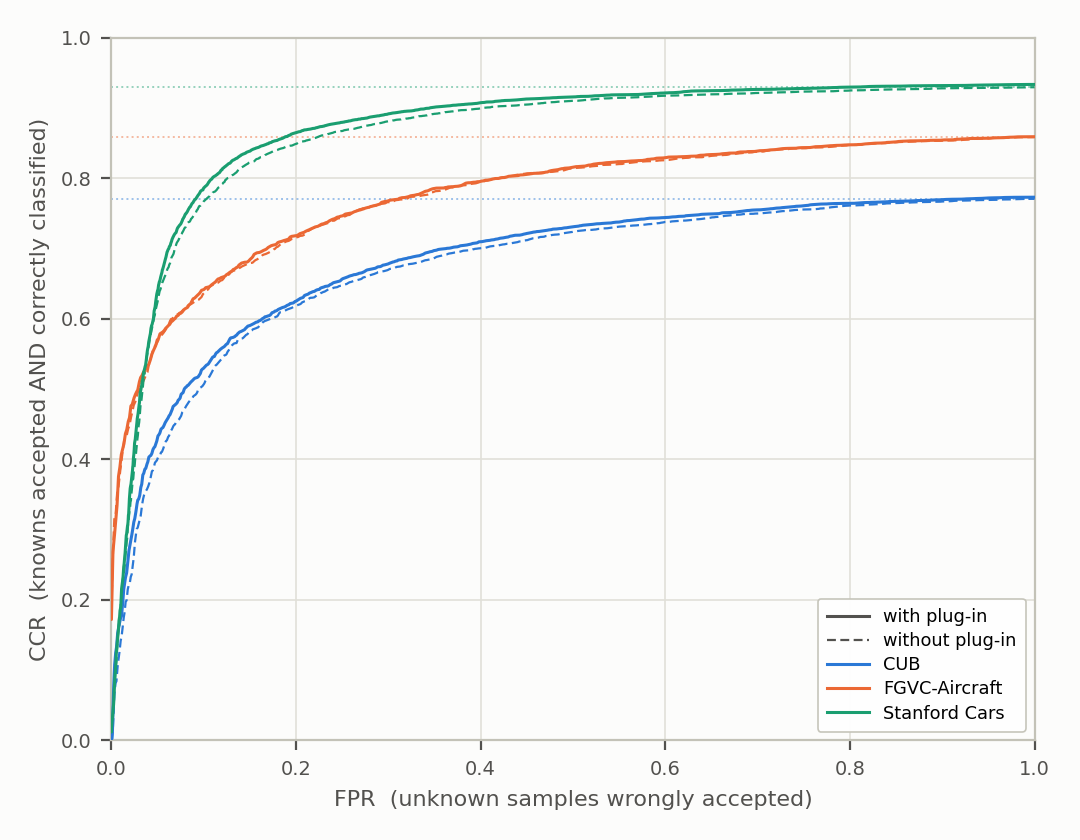}
\end{minipage}\hfill
\begin{minipage}[t]{\oscrpanelw}\centering
{\scriptsize\textbf{\supconmix{}} --- Hard Split}\par\vspace{1pt}
\includegraphics[width=\textwidth]{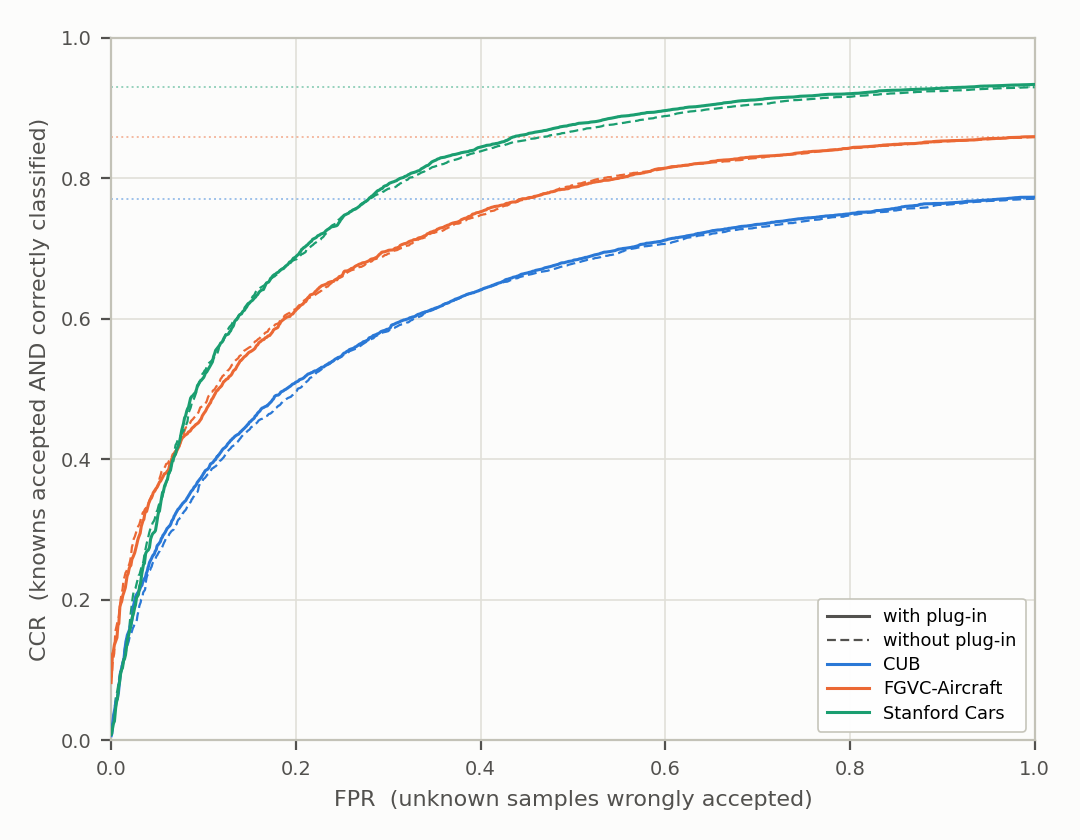}
\end{minipage}
\par\smallskip

\caption{\textbf{OSCR operating curves for \supconmix{} against SSB with the MSP post-processor.}   \textbf{Solid} indicates with \name{} ; \textbf{dashed} is the baseline without the plug-in. The curves are vertically averaged over 3 seeds for a fixed FPR.  The results are shown for both the Easy and Hard Splits.  \name{} pushes out baselines almost consistently across all operating points.}
\label{fig:oscr_grid_a2}
\end{figure}

%OSCR curves, SupCon, SupCon $+$ mixup.} Correct-classification rate of accepted knowns (CCR) against the false-acceptance rate of unknowns (FPR), swept over the MSP score threshold. \textbf{Solid} (`with plug-in' in each key) is the selected $\alpha$; \textbf{dashed} (`without plug-in') is $\alpha{=}0$, the same backbone with the plug-in switched off. Curves are means over 3 seeds, averaged vertically at fixed FPR, so the area under each is exactly the seed-mean OSCR of the corresponding MSP row of Table~\ref{tab:head_oscr_x}. Columns are the two unknown splits (Hard $=$ Hard $+$ Medium). The cross-dataset selected $\alpha^{*}_{\times}$, by arm: \emph{SupCon} CUB~$-0.5$, FGVC-Aircraft~$-4$, Stanford Cars~$-3$; \emph{SupCon $+$ mixup} CUB~$-3$, FGVC-Aircraft~$-1.5$, Stanford Cars~$-4$. Dotted horizontals are each dataset's closed-set accuracy ceiling: a curve flattening below $1$ is that ceiling, not saturation. The remaining arms are in the companion figure.

\begin{figure}[p]
\centering

\begin{minipage}[t]{\oscrpanelw}\centering
{\scriptsize\textbf{\sphor{}} --- Easy Split}\par\vspace{1pt}
\includegraphics[width=\textwidth]{fig_oscr_SupCon_easy.png}
\end{minipage}\hfill
\begin{minipage}[t]{\oscrpanelw}\centering
{\scriptsize\textbf{ \sphor{} } --- Hard Split}\par\vspace{1pt}
\includegraphics[width=\textwidth]{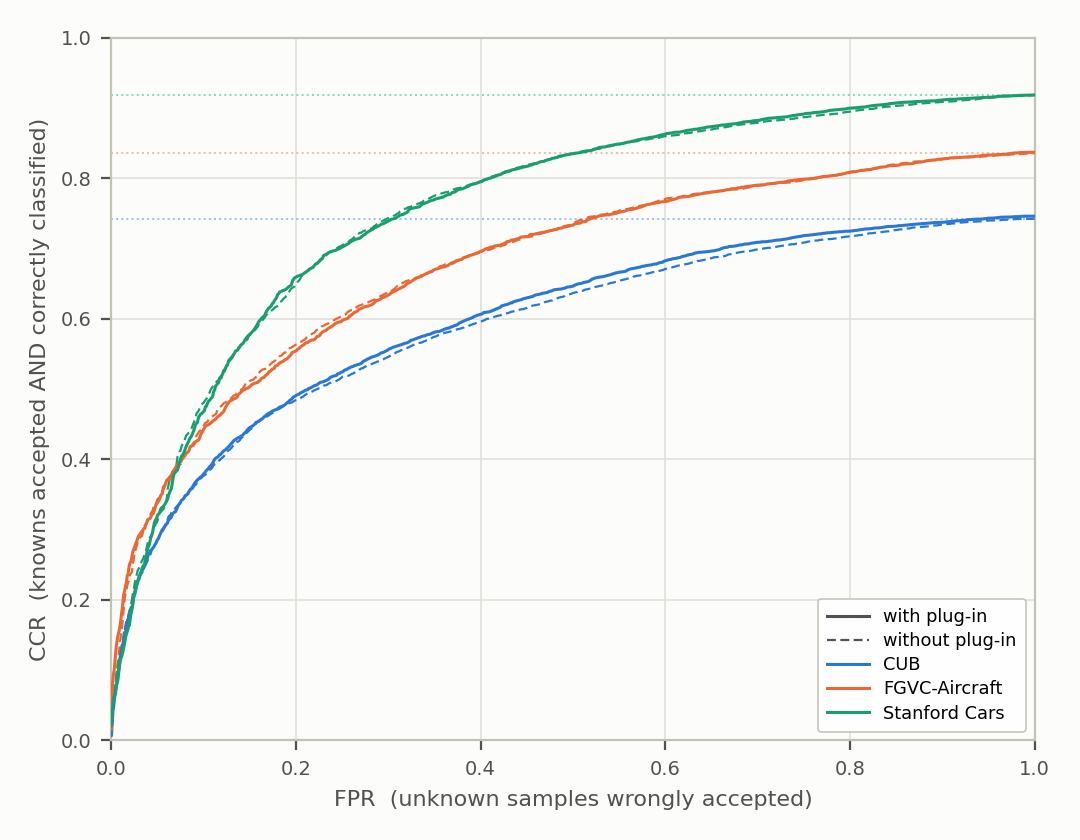}
\end{minipage}
\par\smallskip
\caption{\textbf{OSCR operating curves for \sphor{} against SSB with the MSP post-processor.}   \textbf{Solid} indicates with \name{} ; \textbf{dashed} is the baseline without the plug-in. The curves are vertically averaged over 3 seeds for a fixed FPR.  The results are shown for both the Easy and Hard Splits.  \name{} pushes out baselines almost consistently across all operating points.}
\label{fig:oscr_grid_b1}
\end{figure}

\begin{figure}[p]
\centering
\begin{minipage}[t]{\oscrpanelw}\centering
{\scriptsize\textbf{\sphormix{}} --- Easy Split}\par\vspace{1pt}
\includegraphics[width=\textwidth]{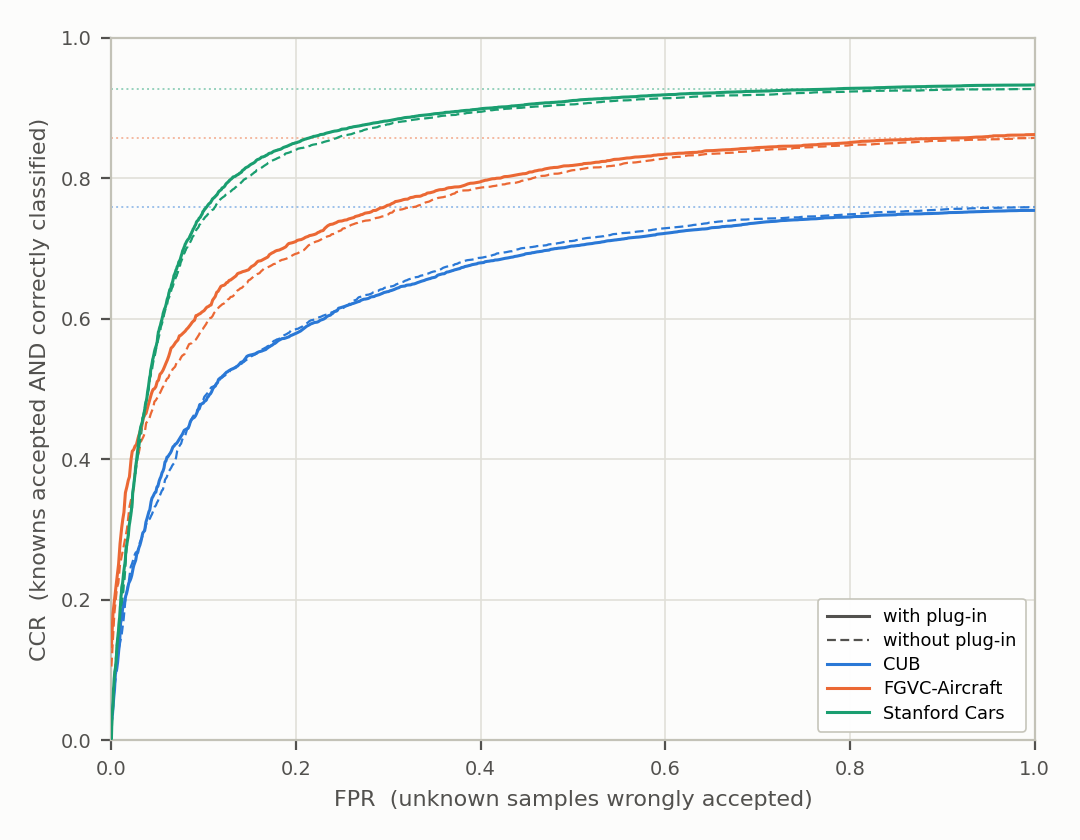}
\end{minipage}\hfill
\begin{minipage}[t]{\oscrpanelw}\centering
{\scriptsize\textbf{\sphormix{}} --- Hard Split}\par\vspace{1pt}
\includegraphics[width=\textwidth]{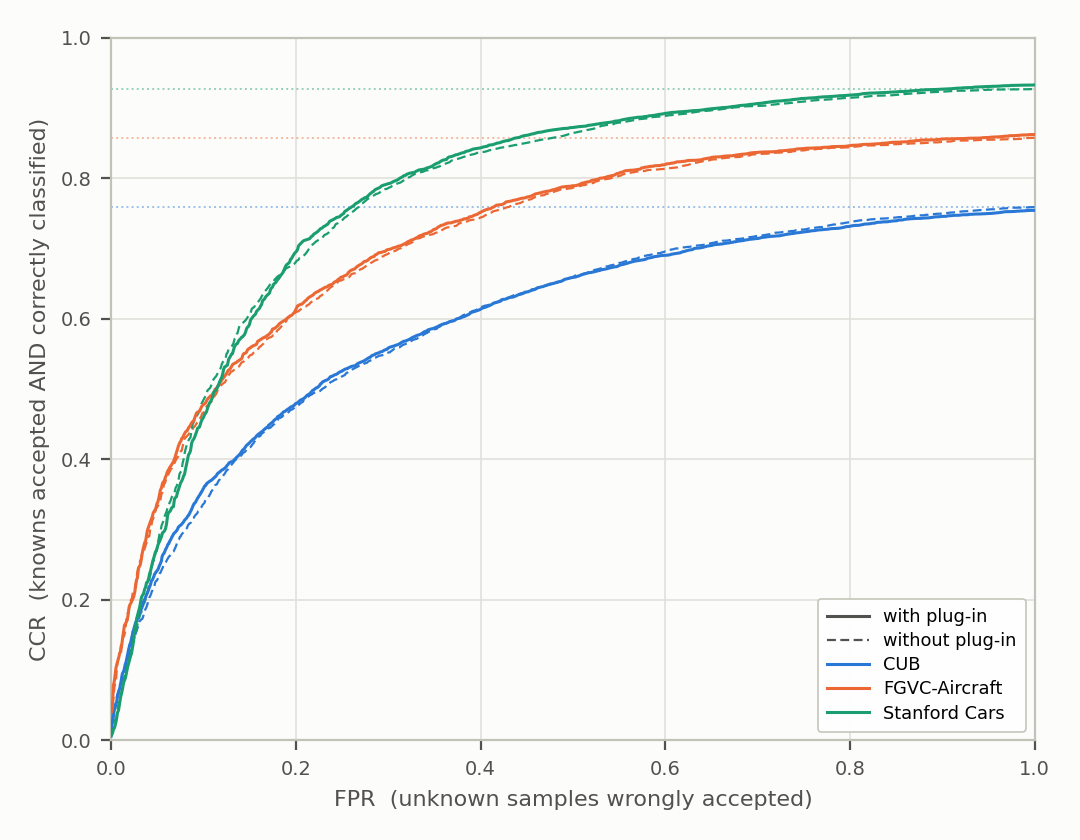}
\end{minipage}
\par\smallskip

\caption{\textbf{OSCR operating curves for \sphormix{} against SSB with the MSP post-processor.}   \textbf{Solid} indicates with \name{} ; \textbf{dashed} is the baseline without the plug-in. The curves are vertically averaged over 3 seeds for a fixed FPR.  The results are shown for both the Easy and Hard Splits.  \name{} pushes out baselines almost consistently across all operating points.}
\label{fig:oscr_grid_b2}
\end{figure}

\noindent \textbf{Robustness Across a Wide Range of Operating Thresholds: }The Open-Set Classification Rate curve characterizes the trade-off between accurately classifying known-class samples ($CCR$) and correctly rejecting unknowns ($FPR$) across varying operational thresholds $\theta$. As illustrated in Fig.~\ref{fig:oscr_grid_a1}, ~\ref{fig:oscr_grid_a2}, ~\ref{fig:oscr_grid_b1},~\ref{fig:oscr_grid_b2}, our ``Enhanced" variants reliably shift the OSCR curves upward for the hard sample splits. This indicates that by focusing on hard positives during training, the model maintains a higher $CCR$ at lower $FPR$ levels. This consistent performance gain across thresholds confirms that our strategy effectively improves open-set recognition. This further validates our hypothesis.

\subsection{FLOPs for \observerclassifier{}} \label{app:flopsanalysis}

\begin{table}[!t]
\centering
\scriptsize
\caption{Performance cost of the \observerclassifier{} on the ResNET-50 backbone. For heavier backbones, the cost will become even smaller. } 
\begin{tabular}{lrrrr}
\toprule
Dataset & $C$ & Head params & \% of backbone params & \% of train step \\
\midrule
CUB & 100 & 204\,900 & 0.80\% & 0.010\% \\
FGVC-Aircraft & 50 & 102\,450 & 0.40\% & 0.005\% \\
Stanford Cars & 98 & 200\,802 & 0.79\% & 0.010\% \\
\bottomrule
\end{tabular}
\label{tab:FLOPS}
\end{table}

It's important to quantify the performance tax when augmenting the conventional training with \name{}. Given that the only difference is that an additional \observerclassifier{} which supplies the $u_i$ is training in parallel with the main-backbone using the cross-entropy loss. Note that this is independent of the any hyper-parameters. Table~\ref{tab:FLOPS} reports the training cost regardless in terms of FLOPs. As observed, the cost is almost negligible.

\section{Extended Ablations and Discussion}
\label{app:extended_ablations}

\subsection{Ablation: Inter-batch and Intra-batch variation decomposed.}

\name{} varies the amount of smoothing along two axes: (1)~Intra-batch,
across the samples in a batch at a given step; and (2)~Inter-batch,
across batches over the course of training. The inter-batch component
is simply the average smoothing per batch, which acts as a label
smoothing schedule. The question for this ablation is whether the
intra-batch component matters, or whether the schedule alone accounts
for the gains. Recall the per-sample coefficient from
Section~\ref{sec: engineering}:

\begin{equation}
\varepsilon_i \;=\; \varepsilon \cdot 2\operatorname{sigmoid}(-|\alpha| u_i)
\tag{\ref{eq:lssubeq}}
\end{equation}

Here $\varepsilon$ and $\alpha$ are shared across the batch, so $u_i$ is
the only source of per-sample variation in $\varepsilon_i$. The
intra-batch effect therefore comes down to which sample receives which
coefficient. To test it, we compare the standard method (treatment)
against two controls that take the computed $\varepsilon_i$ and
redistribute them within the batch before the loss is computed.

\begin{description}
    \item[Treatment] This is the standard \name{} algorithm. We denote $P^{treat}$ as the test performance (OSCR) of a model trained under this regime. 
    \item[\meanfield{}] Instead of the per-sample $\varepsilon_i$, every sample in the batch receives the batch mean: $\varepsilon_i^{\text{bm}} = \frac{1}{B}\sum_j \varepsilon_j$.  We denote $P^{batch-mean}$ as the test performance (OSCR) of a model trained under this regime. We also define $\Delta^{batch-mean} = P^{treat} - P^{batch-mean}$, which measures how much of the gain is from the intra-batch variation and assignment. 

     \item[\reverse{}] \meanfield{} removes the intra-batch variation
    entirely, so it cannot tell us whether the variation helps because the
    right samples get the right coefficients. To test this, we keep each
    batch's coefficients but reverse their order:
    $\varepsilon_i^{\text{rev}} = \varepsilon_{(n+1-r_i)}$, where $r_i$ is
    the rank of $\varepsilon_i$ and $\varepsilon_{(k)}$ is the $k$-th
    smallest coefficient. This arrangement preserves the batch-mean effect and the intra-batch variation, but inverts the assignment. We denote $P^{reverse}$ as the test performance (OSCR) of a model trained under this regime. We also define $\Delta^{reverse} = P^{treat} - P^{reverse}$, which measures how much the correct assignment matters compared to the inverted one.

\end{description}

In all three regimes, all other hyperparameters, base coefficients, and
seeds are kept the same as in standard training.

If the per-sample effect exists and is not a confound of the batch-mean
schedule, we should see positive gains $\Delta^{\text{rev}}$ and
$\Delta^{\text{bm}}$. Ideally, we expect that $\Delta^{\text{rev}} >
\Delta^{\text{bm}} > 0$. Removing the assignment should hurt, but
inverting it should hurt more.

As observed in Table~\ref{tab:mech_compact_avg4}, for method--dataset
combinations, the gains averaged across post-processors are positive,
which indicates that the effect exists by and large.
We observe that our prediction $\Delta^{\text{rev}} > \Delta^{\text{bm}}
> 0$ holds on average ($+0.25 > +0.18$ on easy, $+0.34 > +0.05$ on
hard). Inverting the assignment is more harmful to the method than
removing it.  However, both \reverse{} and \meanfield{} have
the same mean statistics. It must be noted that the negative gains are
not uniformly distributed but concentrated in some methods. For example,
for \sphor{} (without Mixup), 5 out of 6 blocks for $\Delta^{\text{bm}}$
are negative and can be attributed purely to the smoothing schedule;
however, in practice \sphormix{} is used over \sphor{}, since the latter
is an ablation.

\begin{table}[t]
\centering
\small
\setlength{\tabcolsep}{4pt}
\begin{tabular}{l l rrrr rr}
\toprule
 & & \multicolumn{4}{c}{OSCR, mean of $3$ seeds} & \multicolumn{2}{c}{treatment $-$ control} \\
\cmidrule(lr){3-6}\cmidrule(lr){7-8}
arm & data & Baseline & +\name{} & +Batch-mean & +Rev.As. & $\Delta^{Batch-mean}$ & $\Delta^{reverse}$ \\
\midrule
\multicolumn{8}{l}{\textbf{easy}} \\
\supcon & CUB & $59.67$ & $60.83$ & $60.46$ & $60.58$ & \textcolor{gainred}{$+0.37$} & \textcolor{gainred}{$+0.25$} \\
 & Aircraft & $70.87$ & $72.12$ & $71.91$ & $70.88$ & \textcolor{gainred}{$+0.22$} & \textcolor{gainred}{$+1.25$} \\
 & Cars & $82.40$ & $82.92$ & $82.97$ & $82.69$ & $-0.05$ & \textcolor{gainred}{$+0.23$} \\
\supconmix & CUB & $64.05$ & $63.67$ & $63.81$ & $64.20$ & $-0.14$ & $-0.52$ \\
 & Aircraft & $77.17$ & $78.11$ & $76.64$ & $77.57$ & \textcolor{gainred}{$+1.47$} & \textcolor{gainred}{$+0.54$} \\
 & Cars & $84.69$ & $85.42$ & $84.96$ & $84.83$ & \textcolor{gainred}{$+0.45$} & \textcolor{gainred}{$+0.59$} \\
\sphor & CUB & $63.67$ & $64.64$ & $65.01$ & $64.77$ & $-0.37$ & $-0.13$ \\
 & Aircraft & $69.62$ & $69.81$ & $70.06$ & $69.27$ & $-0.25$ & \textcolor{gainred}{$+0.55$} \\
 & Cars & $84.00$ & $84.17$ & $84.09$ & $83.96$ & \textcolor{gainred}{$+0.07$} & \textcolor{gainred}{$+0.21$} \\
\sphormix & CUB & $66.76$ & $67.56$ & $67.30$ & $67.79$ & \textcolor{gainred}{$+0.26$} & $-0.23$ \\
 & Aircraft & $77.67$ & $77.89$ & $78.08$ & $78.08$ & $-0.19$ & $-0.19$ \\
 & Cars & $85.64$ & $86.51$ & $86.16$ & $86.07$ & \textcolor{gainred}{$+0.36$} & \textcolor{gainred}{$+0.44$} \\
\addlinespace[2pt]
\multicolumn{6}{l}{\emph{mean over 12 blocks}} & \textcolor{gainred}{$+0.18$} & \textcolor{gainred}{$+0.25$} \\

\midrule
\multicolumn{8}{l}{\textbf{hard}} \\
\supcon & CUB & $54.37$ & $55.01$ & $54.93$ & $54.38$ & \textcolor{gainred}{$+0.08$} & \textcolor{gainred}{$+0.63$} \\
 & Aircraft & $66.16$ & $68.60$ & $67.81$ & $67.68$ & \textcolor{gainred}{$+0.79$} & \textcolor{gainred}{$+0.92$} \\
 & Cars & $75.13$ & $75.17$ & $76.05$ & $75.25$ & $-0.88$ & $-0.08$ \\
\supconmix & CUB & $57.06$ & $57.11$ & $57.05$ & $57.32$ & \textcolor{gainred}{$+0.06$} & $-0.21$ \\
 & Aircraft & $70.72$ & $71.29$ & $70.35$ & $70.34$ & \textcolor{gainred}{$+0.94$} & \textcolor{gainred}{$+0.95$} \\
 & Cars & $77.44$ & $77.84$ & $77.68$ & $77.26$ & \textcolor{gainred}{$+0.17$} & \textcolor{gainred}{$+0.59$} \\
\sphor & CUB & $57.62$ & $58.43$ & $58.57$ & $57.95$ & $-0.14$ & \textcolor{gainred}{$+0.48$} \\
 & Aircraft & $67.28$ & $67.27$ & $67.49$ & $66.93$ & $-0.22$ & \textcolor{gainred}{$+0.34$} \\
 & Cars & $75.64$ & $75.71$ & $76.17$ & $75.76$ & $-0.46$ & $-0.05$ \\
\sphormix & CUB & $59.47$ & $60.34$ & $60.19$ & $60.23$ & \textcolor{gainred}{$+0.15$} & \textcolor{gainred}{$+0.11$} \\
 & Aircraft & $70.85$ & $71.00$ & $71.05$ & $70.98$ & $-0.04$ & \textcolor{gainred}{$+0.03$} \\
 & Cars & $77.99$ & $78.44$ & $78.25$ & $78.05$ & \textcolor{gainred}{$+0.20$} & \textcolor{gainred}{$+0.40$} \\
\addlinespace[2pt]
\multicolumn{6}{l}{\emph{mean over 12 blocks}} & \textcolor{gainred}{$+0.05$} & \textcolor{gainred}{$+0.34$} \\
\bottomrule
\end{tabular}
\caption{ OSCR
averaged over MSP, MLS, NNGuide-5 and PostMax across $3$ seeds. The three regimes are \meanfield{} (+Batch-mean), \reverse{} (+Rev.As.) and Treatment (+\name{}). To isolate the effectiveness of the per-sample assignment, we calculate two gains by subtracting the \reverse{}/\meanfield{} regime from the \name{} regime, which highlights the gain in terms of the per-sample smoothing. If the gain is positive, it is marked in blue and is supportive of our hypothesis. }
\label{tab:mech_compact_avg4}
\end{table}

\subsection{Temporal Behavior of the Smoothing Method}

\label{sec:dose}
In this section, we are interested in tracking the behavior of $\varepsilon_i$, and specifically how it behaves in a temporal sense. Note that $\varepsilon_i$ is not stationary, as it is a function of 
the model's own \quantity{} $u_i$. As the model trains, it becomes more confident, and $u$ drifts toward $-1$, and $\varepsilon$ is boosted. To observe how $\varepsilon$ behaves, we plot the two information in a temporal chart: (a) the batch-mean value of $\varepsilon$; (b) inter-decile ($p_{10}$--$p_{90}$) range across the batch for $\varepsilon$. We use this instead of a standard deviation because the distribution is skewed towards the upper bound ($\varepsilon$ is bounded). This is illustrated for all method-dataset combinations in Fig.~\ref{fig:epstraj}.

\begin{figure}[t]
\centering
\includegraphics[width=\textwidth]{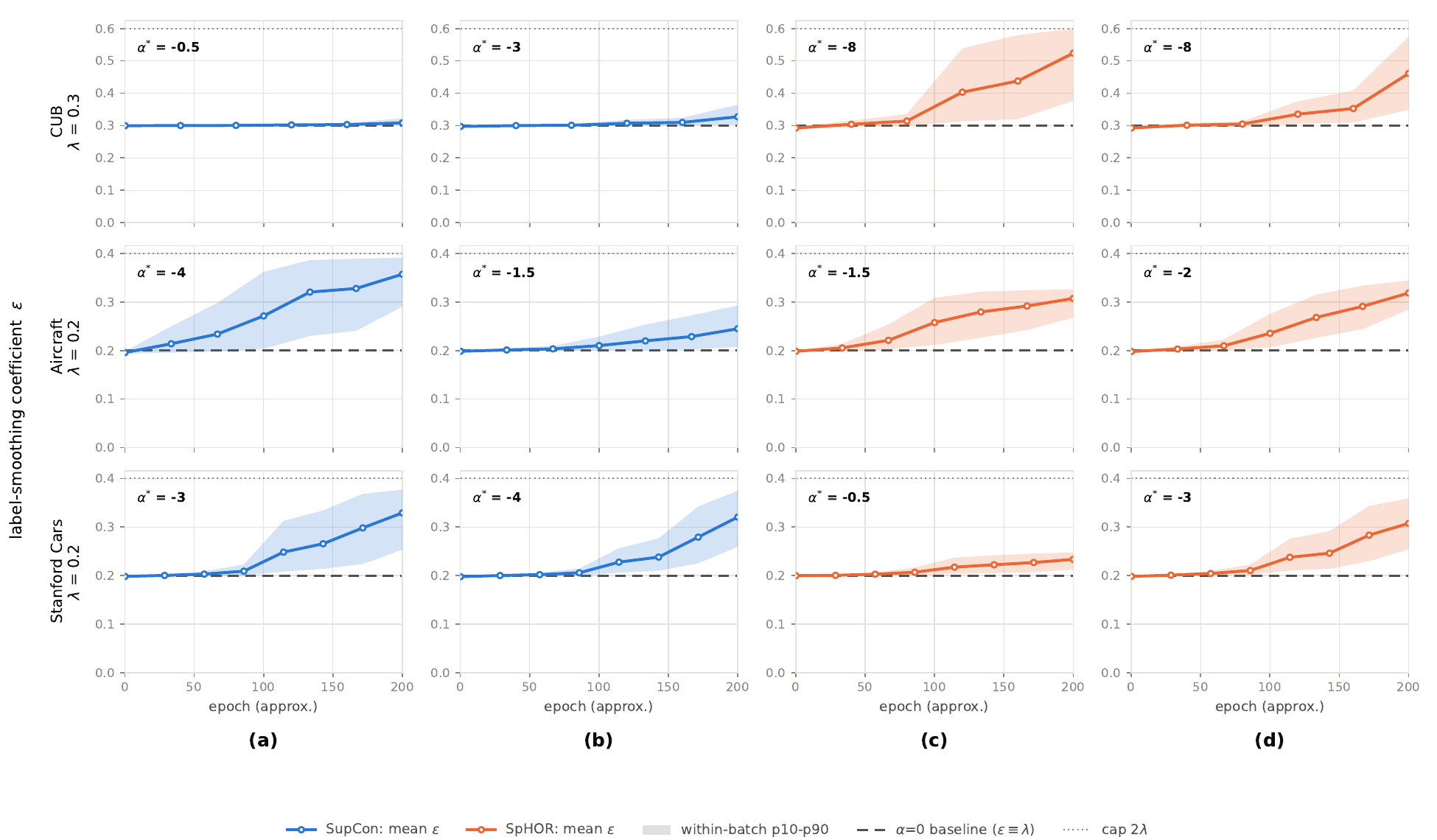}
\caption{Evolution of the Label Smoothing coefficient $\varepsilon$ across the training epochs across all dataset-method combinations, \textbf{(a)} \supcon{}, \textbf{(b)} \supconmix{}, \textbf{(c)} \sphor{}, \textbf{(d)} \sphormix{} , and rows are the three datasets, each annotated with its inherited $\lambda$. Results averaged over 3 seeds. The line is the batch-mean of $\varepsilon$ and the band is \textbf{inter-decile ($p_{10}$--$p_{90}$) range across the batch}. The dashed line is the $\alpha=0$ baseline. The X-axis is the time mapped to the training epochs}.
\label{fig:epstraj}
\end{figure}

Based on this chart, we can come with the following conclusions.  The \name{} does not give a constant intervention during the training. It is time-varying and gradually increases. For example, for \sphormix{}/CUB it goes from $\lambda=0.3$ to $0.524$. This is an expected behavior as the model becomes more confident. From another point of view, this can also be seen as a form of regularization as the model becomes more confident. In addition to an adaptive per-sample smoothing method, the method also functions as a time-varying smoothing schedule.

% ---------------------------------------------------------------------------
\subsection{Ablation: Mixing smoothing coefficients vs mixing margins.} \label{sec:epsblend}
In Section~\ref{section:epsblender}, we use the following scheme, which we refer to as \textbf{Margin Blending}, for Mixup.
The constituent margins are read from the forward pass and retrieved through
the mixing indices:
\begin{equation}
\text{u}^{\lambda}_{i,j} = \lambda \cdot \text{u}_{i} + (1 - \lambda) \cdot \text{u}_{j}
\end{equation}
This composite margin induces a single smoothing coefficient, applied once to the mixed
target:
\begin{equation}
\varepsilon^{\lambda}_{i,j} \;=\; 2\varepsilon \cdot
   \operatorname{sigmoid}\!\big(-|\alpha|\,\text{u}^{\lambda}_{i,j}\big),
\qquad \varepsilon^{\lambda}_{i,j} \in (0,\, 2\varepsilon)
\label{eq:epsmix0}
\end{equation}
\begin{equation}
\hat{\mathbf{y}}_{k}^{'} =
  (1 - \varepsilon^{\lambda}_{i,j})\big(\lambda \mathbf{y}_{i} + (1-\lambda)\mathbf{y}_{j}\big)
  + \frac{\varepsilon^{\lambda}_{i,j}}{C} \cdot \mathbf{1},
\label{eq:lsmix2}
\end{equation}
where $\mathbf{y}_{i}, \mathbf{y}_{j}$ are the \emph{one-hot} targets of the constituent
samples. 

However, it can be argued that the coefficients can be blended as well. Thus, we ablate it against an alternative scheme known as \textbf{coefficient blending:} 
\begin{equation}
\text{u}^{\lambda}_{i,j} = \lambda \cdot \text{u}_{i} + (1 - \lambda) \cdot \text{u}_{j}
\end{equation}
This composite margin induces a single smoothing coefficient, applied once to the mixed
target:
\begin{equation}
\varepsilon_{i} \;=\; 2\varepsilon \cdot
   \operatorname{sigmoid}\!\big(-|\alpha|\,\text{u}_{i}\big),
\qquad \varepsilon_{i} \in (0,\, 2\varepsilon)
\label{eq:epsmix1}
\end{equation}
\begin{equation}
\varepsilon_{j} \;=\; 2\varepsilon \cdot
   \operatorname{sigmoid}\!\big(-|\alpha|\,\text{u}_{j}\big),
\qquad \varepsilon_{j} \in (0,\, 2\varepsilon)
\label{eq:epsmix2}
\end{equation}
\begin{equation}
\varepsilon^{\lambda}_{i,j} \;=\; \lambda \cdot \varepsilon_{i} + (1 - \lambda) \cdot \varepsilon_{j},
\qquad \varepsilon^{\lambda}_{i,j} \in (0,\, 2\varepsilon)
\label{eq:epsmix3}
\end{equation}
\begin{equation}
\hat{\mathbf{y}}_{k}^{'} =
  (1 - \varepsilon^{\lambda}_{i,j})\big(\lambda \mathbf{y}_{i} + (1-\lambda)\mathbf{y}_{j}\big)
  + \frac{\varepsilon^{\lambda}_{i,j}}{C} \cdot \mathbf{1},
\label{eq:lsmix3}
\end{equation}
where $\mathbf{y}_{i}, \mathbf{y}_{j}$ are the \emph{one-hot} targets of the constituent
samples. 

As observed in Table~\ref{tab:epsblend}, \textbf{Margin blending wins.}  It is ahead on $11$ of the $12$ cells in all dataset-method combinations in OSCR Hard, and by $+0.27$ OSCR points on both splits. The formula we used for blending margins turns out to be better in practice, but the success of our method doesn't depend on it.

\begin{table}[h]
\centering
\scriptsize
\begin{tabular}{ll rrr rrr}
\toprule
& & \multicolumn{3}{c}{OSCR-Easy} & \multicolumn{3}{c}{OSCR-Hard} \\
\cmidrule(lr){3-5}\cmidrule(lr){6-8}
Objective & Dataset & coeff. & margin & $\Delta$ & coeff. & margin & $\Delta$ \\
\midrule
\supconmix & CUB & $65.41$ & $65.19$ & $-0.23$ & $59.72$ & $59.86$ & \textcolor{gainred}{$+0.14$} \\
\supconmix & Aircraft & $76.56$ & $76.94$ & \textcolor{gainred}{$+0.38$} & $71.70$ & $72.09$ & \textcolor{gainred}{$+0.38$} \\
\supconmix & Cars & $86.05$ & $86.13$ & \textcolor{gainred}{$+0.08$} & $78.05$ & $78.17$ & \textcolor{gainred}{$+0.12$} \\
\sphormix & CUB & $67.80$ & $68.26$ & \textcolor{gainred}{$+0.45$} & $61.98$ & $62.20$ & \textcolor{gainred}{$+0.21$} \\
\sphormix & Aircraft & $76.85$ & $77.55$ & \textcolor{gainred}{$+0.70$} & $71.50$ & $72.04$ & \textcolor{gainred}{$+0.53$} \\
\sphormix & Cars & $86.92$ & $87.13$ & \textcolor{gainred}{$+0.21$} & $78.80$ & $79.05$ & \textcolor{gainred}{$+0.25$} \\
\midrule
\multicolumn{2}{l}{Mean over the six arms} & $76.60$ & $76.86$ & \textcolor{gainred}{$+0.27$} & $70.29$ & $70.57$ & \textcolor{gainred}{$+0.27$} \\
\bottomrule
\end{tabular}
\caption{\textbf{Margin blending beats coefficient blending on $11$ of $12$ cells.} OSCR
under MSP at shared hyperparameters, mean over $3$ seeds; $\Delta=\text{margin blending}-\text{coeff. blending}$, red where margin blending is ahead.}
\label{tab:epsblend}
\end{table}

\subsection{How does \name{} optimize representation geometry for separability?}
\label{app:asandns}

While intra-class alignment and inter-class uniformity assess representations within known classes, they fail to evaluate how effectively feature representations distinguish known classes from unknown ones. To address this, \citet{sphor} proposed Angular Separability (AS) and Norm Separability (NS):

\begin{align}
\mathrm{AS} &\;=\; \frac{1}{|D_{u}|} \sum_{\mathbf{u} \in D_{u}} \max_{\mathbf{v} \in D_{k}} \left( \frac{\mathbf{u} \cdot \mathbf{v}}{\|\mathbf{u}\| \|\mathbf{v}\|}\right) \qquad (\downarrow), \label{eq:as} \\[4pt]
\mathrm{NS} &\;=\; \mathrm{AUROC}\bigl( \{\|\mathbf{v}\| : \mathbf{v} \in D_k\}, \{\|\mathbf{u}\| : \mathbf{u} \in D_u\} \bigr) \qquad (\uparrow). \label{eq:ns}
\end{align}

\textbf{Angular Separability (AS)} (Eq.~\ref{eq:as}) measures the average cosine similarity between each unknown sample ($\mathbf{u} \in D_{u}$) and its nearest known-class counterpart ($\mathbf{v} \in D_{k}$) in the test set \citep{sphor}. It quantifies the degree to which unknown samples fall into the "familiarity trap," where lower values indicate better angular separation between known and unknown classes.

\textbf{Norm Separability (NS)} (Eq.~\ref{eq:ns}) evaluates the AUROC obtained by using feature vector norms to distinguish known-class test samples from unknown-class samples \citep{sphor}. Higher values signify that unknown classes can be easily detected via a simple, class-agnostic threshold on feature magnitudes.

\begin{table}[t]
\centering
\small
\setlength{\tabcolsep}{4pt}
\begin{tabular}{l l rrr rrr}
\toprule
 & & \multicolumn{3}{c}{AS $\downarrow$} & \multicolumn{3}{c}{NS $\uparrow$} \\
\cmidrule(lr){3-5}\cmidrule(lr){6-8}
arm & data & w/o & $+$\name{} & gain & w/o & $+$\name{} & gain \\
\midrule
\multicolumn{8}{l}{\textbf{easy}} \\
\supcon{} & CUB & $\mathbf{0.8793}$ & $0.8796$ & $-0.0003$ & $69.22$ & $\mathbf{72.21}$ & \textcolor{gainred}{$+2.99$} \\
 & Aircraft & $\mathbf{0.9130}$ & $0.9159$ & $-0.0030$ & $77.59$ & $\mathbf{78.07}$ & \textcolor{gainred}{$+0.48$} \\
 & Cars & $\mathbf{0.8959}$ & $0.8987$ & $-0.0027$ & $\mathbf{84.98}$ & $84.14$ & $-0.84$ \\
\supconmix{} & CUB & $0.8771$ & $\mathbf{0.8766}$ & \textcolor{gainred}{$+0.0005$} & $\mathbf{68.94}$ & $68.52$ & $-0.42$ \\
 & Aircraft & $\mathbf{0.8939}$ & $0.8948$ & $-0.0010$ & $68.05$ & $\mathbf{72.62}$ & \textcolor{gainred}{$+4.57$} \\
 & Cars & $0.8804$ & $0.8804$ & $0.0000$ & $79.95$ & $\mathbf{81.98}$ & \textcolor{gainred}{$+2.03$} \\
\sphor{} & CUB & $\mathbf{0.8589}$ & $0.8695$ & $-0.0106$ & $68.24$ & $\mathbf{72.16}$ & \textcolor{gainred}{$+3.92$} \\
 & Aircraft & $\mathbf{0.9002}$ & $0.9061$ & $-0.0060$ & $63.49$ & $\mathbf{65.38}$ & \textcolor{gainred}{$+1.89$} \\
 & Cars & $\mathbf{0.8837}$ & $0.8865$ & $-0.0028$ & $83.25$ & $\mathbf{83.99}$ & \textcolor{gainred}{$+0.74$} \\
\sphormix{} & CUB & $\mathbf{0.8356}$ & $0.8425$ & $-0.0069$ & $56.60$ & $\mathbf{65.35}$ & \textcolor{gainred}{$+8.75$} \\
 & Aircraft & $\mathbf{0.8618}$ & $0.8668$ & $-0.0050$ & $60.29$ & $\mathbf{64.07}$ & \textcolor{gainred}{$+3.78$} \\
 & Cars & $\mathbf{0.8543}$ & $0.8551$ & $-0.0008$ & $63.13$ & $\mathbf{70.54}$ & \textcolor{gainred}{$+7.41$} \\
\addlinespace[2pt]
\multicolumn{4}{l}{\emph{mean over 12 blocks ($+$\name{} ahead)}} & $-0.0032$ (1) & \multicolumn{2}{c}{} & \textcolor{gainred}{$+2.94$} (10) \\
\midrule
\multicolumn{8}{l}{\textbf{hard}} \\
\supcon{} & CUB & $\mathbf{0.8968}$ & $0.9005$ & $-0.0038$ & $65.08$ & $\mathbf{65.72}$ & \textcolor{gainred}{$+0.64$} \\
 & Aircraft & $\mathbf{0.9298}$ & $0.9307$ & $-0.0009$ & $73.71$ & $\mathbf{74.83}$ & \textcolor{gainred}{$+1.12$} \\
 & Cars & $\mathbf{0.9212}$ & $0.9221$ & $-0.0009$ & $75.24$ & $\mathbf{75.90}$ & \textcolor{gainred}{$+0.66$} \\
\supconmix{} & CUB & $0.8952$ & $\mathbf{0.8947}$ & \textcolor{gainred}{$+0.0005$} & $65.66$ & $\mathbf{66.06}$ & \textcolor{gainred}{$+0.40$} \\
 & Aircraft & $\mathbf{0.9135}$ & $0.9177$ & $-0.0042$ & $72.14$ & $\mathbf{72.46}$ & \textcolor{gainred}{$+0.33$} \\
 & Cars & $\mathbf{0.9006}$ & $0.9038$ & $-0.0032$ & $72.58$ & $\mathbf{73.84}$ & \textcolor{gainred}{$+1.26$} \\
\sphor{} & CUB & $\mathbf{0.8736}$ & $0.8838$ & $-0.0102$ & $64.82$ & $\mathbf{66.74}$ & \textcolor{gainred}{$+1.92$} \\
 & Aircraft & $\mathbf{0.9042}$ & $0.9090$ & $-0.0049$ & $69.61$ & $\mathbf{70.37}$ & \textcolor{gainred}{$+0.77$} \\
 & Cars & $\mathbf{0.9031}$ & $0.9063$ & $-0.0032$ & $74.08$ & $\mathbf{74.54}$ & \textcolor{gainred}{$+0.46$} \\
\sphormix{} & CUB & $\mathbf{0.8566}$ & $0.8641$ & $-0.0075$ & $58.39$ & $\mathbf{62.59}$ & \textcolor{gainred}{$+4.20$} \\
 & Aircraft & $\mathbf{0.8800}$ & $0.8847$ & $-0.0047$ & $65.80$ & $\mathbf{68.00}$ & \textcolor{gainred}{$+2.20$} \\
 & Cars & $\mathbf{0.8761}$ & $0.8806$ & $-0.0044$ & $61.49$ & $\mathbf{66.42}$ & \textcolor{gainred}{$+4.93$} \\
\addlinespace[2pt]
\multicolumn{4}{l}{\emph{mean over 12 blocks ($+$\name{} ahead)}} & $-0.0039$ (1) & \multicolumn{2}{c}{} & \textcolor{gainred}{$+1.57$} (12) \\
\bottomrule
\end{tabular}

\caption{Angular Separability (AS, Eq.~\ref{eq:as}, lower is better) and Norm Separability
(NS, Eq.~\ref{eq:ns}, AUROC in $\%$, higher is better) on the test set, mean of $3$ seeds,
without and with $+$\name{}. Within
each pair the better level is bold. Gain is the improvement in the metric's own direction,
w/o $-$ $+$\name{} for AS and $+$\name{} $-$ w/o for NS, paired by seed; a positive gain is
red.}

\label{tab:geom_spread}
\end{table}

Table~\ref{tab:geom_spread} shows a clear trade-off. Adding \name{}
improves norm separability (NS) by $+2.94$ AUROC points on Easy and
$+1.57$ on Hard on average. This comes at a small cost in angular separability (AS), which drops by
less than $0.5\%$ of its value on both splits ($0.0032$ on Easy and
$0.0039$ on Hard, against levels of $0.84$--$0.93$) The NS gain is several times larger than the AS loss, and the net
effect is better unknown rejection.

\subsection{Hyperparameter Analysis}
\label{app:integration}

\textbf{Validation Rule} We design a "validation rule" to find the optimal hyperparameter $\alpha^{*}$ in the absence of a test dataset. We take the training dataset split into 4-fold to get the validation known class set and the training dataset. We train the model across a range of $\alpha^{*} \in \{8,4,3,2,1.5,1,0.5\}$. Afterwards, we use a different fine-grained dataset's images as the unknown class validation set  (CUB$\rightarrow$ Cars $\rightarrow$ Aircraft $\rightarrow$ CUB). After training, for each of the split, using the MSP post-processor, we take the argmax over the OSCR performance with the known and unknown set. 

We pick the optimal $\alpha$ based on the best performance on these curves as shown in Figure~\ref{fig:xval-conosr-MAIJOR-peek} and Figure~\ref{fig:xval-sphor-MAIJOR-peek}.

\begin{figure}[t]
\centering
\includegraphics[width=\textwidth]{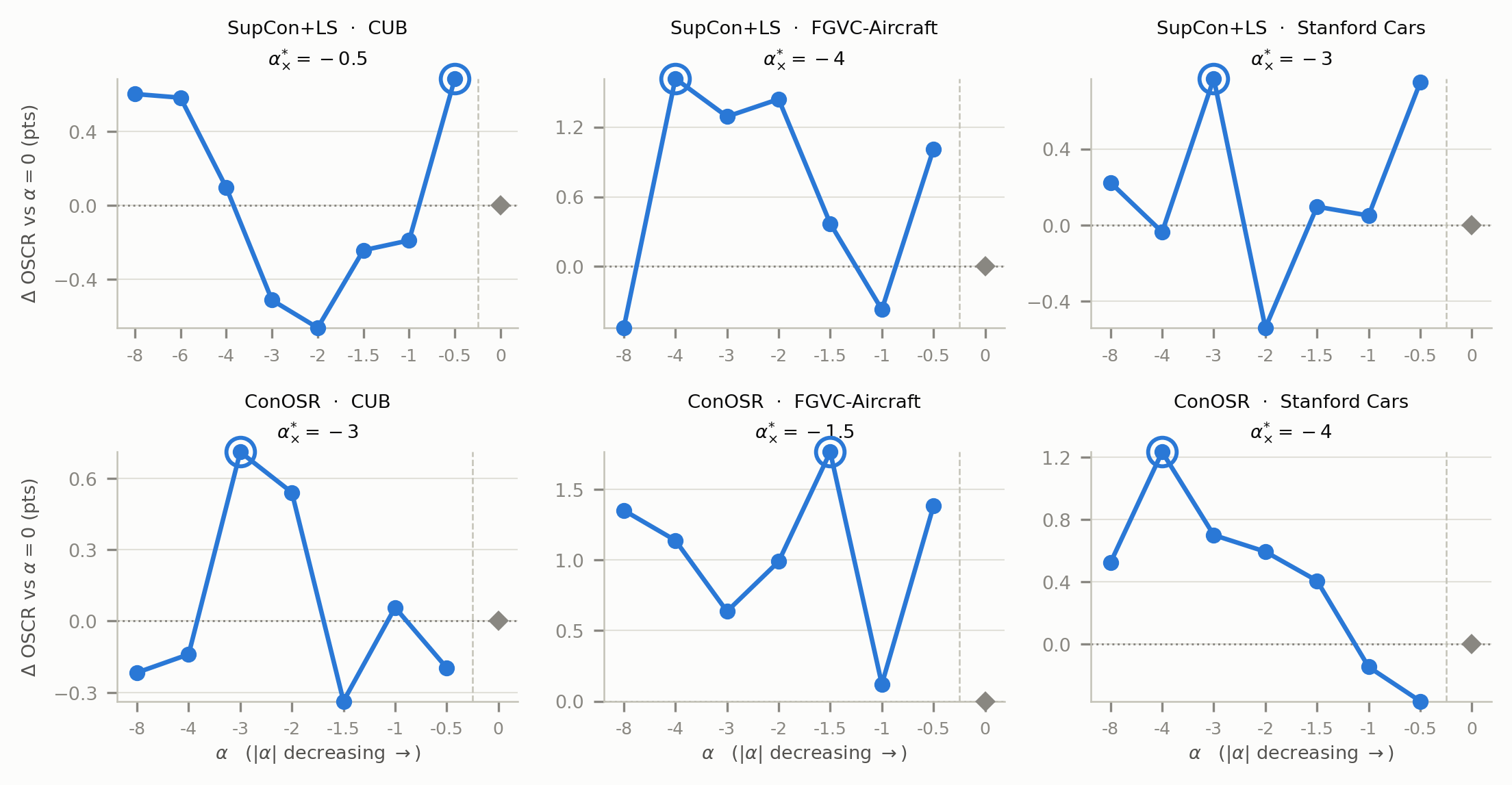}
\caption{\supcon{} and \supconmix{}: Hyperparameter Sweep for $\alpha$
}
\label{fig:xval-conosr-MAIJOR-peek}
\end{figure}

\begin{figure}[t]
\centering
\includegraphics[width=\textwidth]{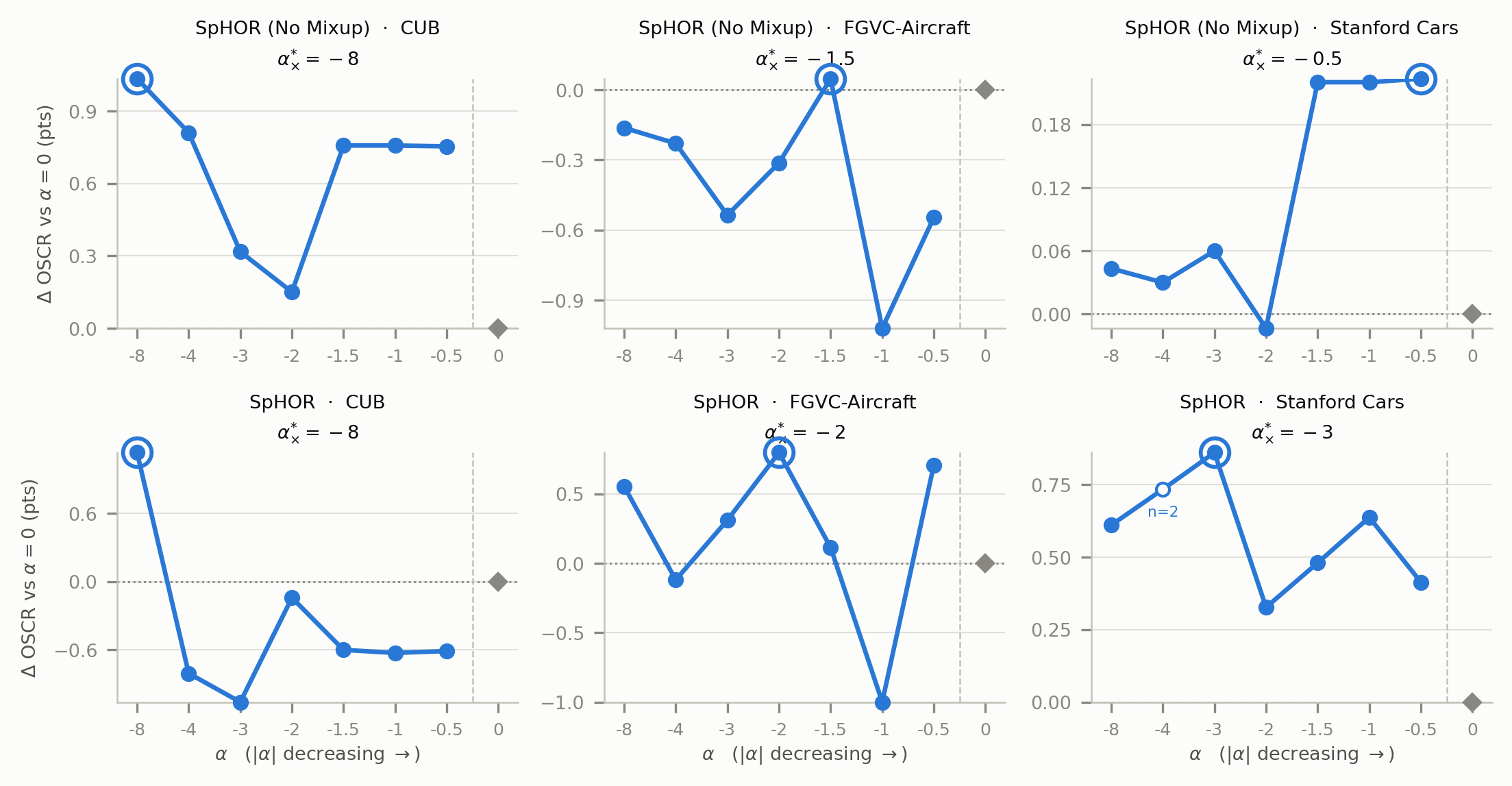}
\caption{\sphor{} and \sphormix{}: Hyperparameter Sweep for $\alpha$}
\label{fig:xval-sphor-MAIJOR-peek}
\end{figure}

% =====================================================================
\section{Methodology: Integration}
\label{app:repl_integration}

We first state \name{} as a mechanism on the label space, then the four base methods in their published form (Section~\ref{sec:base}), and finally the effect of the smoothing factor $\varepsilon_i$ on each of them.   Every base method is recovered exactly at $\varepsilon_i \equiv 0$.

\name{} is a transformation on the label space $\Delta^{C-1}$, applied independently of the loss function. For a sample $i$ with smoothing factor $\varepsilon_i \in [0,1]$:
\begin{equation}
\mathbf{y}_i' := \mathrm{\name{}}(\mathbf{y}_i) = (1 - \varepsilon_i)\,\mathbf{y}_i + \frac{\varepsilon_i}{C} \cdot \mathbf{1}, \qquad \varepsilon_i \in [0,1],
\label{eq:ls2}
\end{equation}
where $\mathbf{1} \in \mathbbm{R}^C$ is a vector of ones. 

This transformation has two key properties:

\textbf{Simplex closure:} If $\mathbf{y}_i \in \Delta^{C-1}$, then $\mathbf{y}_i' \in \Delta^{C-1}$ always sums to $1$:
\begin{equation}
\sum_{k=1}^C y_{ik}' = 1
\label{eq:simplex}
\end{equation}

\textbf{Endpoints:} The transformation reduces to the original one-hot vector or the uniform distribution at its limits:
\begin{equation}
\mathbf{y}_i'\big|_{\varepsilon_i = 0} = \mathbf{y}_i, \qquad \mathbf{y}_i'\big|_{\varepsilon_i = 1} = \frac{1}{C}\mathbf{1}
\label{eq:endpoints}
\end{equation}

\subsection{Base Methods}
\label{sec:base}

\subsubsection{\sphormix{}}
\label{intsphor}

The \sphormix{} framework employs the Von Mises-Fisher Alignment Loss ($\mathcal{L}_{\mathrm{vMFAL}}$)
to map features onto a hyperspherical embedding space. For a dataset of $N$ samples
$(x_i, y_i)$, the loss is defined as:
\begin{equation}
\mathcal{L}_{\mathrm{vMFAL}} := - \frac{1}{N} \sum_{i=1}^N \underbrace{\sum_{k=1}^C S_{ik} \log P_{ik}}_{=: \mathcal{L}_{\mathrm{vMFAL}}^{(i)}} %+ R(\{\boldsymbol{\mu}_k\}_{k=0}^{C})
\label{eq:vmfal}
\end{equation}
where the predicted probability $P_{ik}$ for sample $i$ and class $k$ is given by the centered
cosine similarity:
\begin{equation}
P_{ik} = \frac{\exp(z_i^\top \boldsymbol{\mu}_k / \tau)}{\sum_{j=1}^C \exp(z_i^\top \boldsymbol{\mu}_j / \tau)}
\label{eq:Pik}
\end{equation}
In this formulation, $\boldsymbol{\mu}_k$ denotes the label
embedding for class $k$, and $\tau$ is a temperature scaling parameter. Both the latent feature
vector $z_i$ and the label embeddings $\boldsymbol{\mu}_k$ are $L_2$-normalized. The term $S_{ik}$
represents the soft label similarity, defined as:
\begin{equation}
S_{ik} = \frac{y_{ik}}{\sum_{j=1}^C y_{ij}} = y_{ik}
\label{eq:sik-base}
\end{equation}

To simulate the open space, \sphormix{} additionally trains on Mixup-generated virtual samples. For a pair $(a,b)$ and $\lambda \sim \mathrm{Beta}(2,2)$,
\begin{equation}
\tilde{\mathbf{x}}_j = \lambda \mathbf{x}_a + (1-\lambda)\mathbf{x}_b,
\qquad
\tilde{\mathbf{y}}_j = \lambda \mathbf{y}_a + (1-\lambda)\mathbf{y}_b \;\in\; \Delta^{C-1},
\label{eq:mixup}
\end{equation}
and $\mathcal{L}_{\mathrm{vMFAL}}^{\text{Mixup}(j)}$ is the per-sample term of \eqref{eq:vmfal} 
evaluated at $(\tilde{z}_j, \tilde{S}_{jk} = \tilde{y}_{jk})$:
\begin{equation}
\mathcal{L}_{\sphormix} :=  -\frac{1}{N} \sum_{i=1}^N \mathcal{L}_{\mathrm{vMFAL}}^{(i)}
\;-\; \frac{1}{N} \sum_{j=1}^N \mathcal{L}_{\mathrm{vMFAL}}^{\text{Mixup}(j)}
\;+\; R(\{\boldsymbol{\mu}_k\}_{k=1}^{C})
\label{eq:sphor}
\end{equation}

$R$ here is a regularization loss based on the label embeddings. 

% ...............................................................
\subsubsection{Ablation without Mixup: \sphor{}}
\label{sphormix}

To verify that our results do not hinge on Mixup, we also report
\sphor{}, which drops the Mixup term from Eq.~\eqref{eq:sphor}:
\begin{equation}
\mathcal{L}_{\sphor} :=  -\frac{1}{N} \sum_{i=1}^N \mathcal{L}_{\mathrm{vMFAL}}^{(i)}+\; R(\{\boldsymbol{\mu}_k\}_{k=1}^{C})
\end{equation}

\subsubsection{Supervised Contrastive Learning}
\label{conosr}

The \supcon{} framework adopts Supervised Contrastive Learning with Soft Targets (SupCon-ST)
as its representation learning objective. For a batch of $N$ samples, given the
$L_2$-normalized projection features $\{z_i\}_{i=1}^N$ obtained after RandAugment, the loss is
defined as:
\begin{equation}
\mathcal{L}_{\mathrm{SupCon\text{-}ST}} := - \frac{1}{N} \sum_{i=1}^N \underbrace{\sum_{j \neq i} S_{ij} \log P_{ij}}_{=: \mathcal{L}_{\mathrm{SupCon\text{-}ST}}^{(i)}}
\label{eq:supconst}
\end{equation}
where the pairwise prediction $P_{ij}$ is the softmax-normalized cosine similarity over all
other samples in the batch:
\begin{equation}
P_{ij} = \frac{\exp(z_i^\top z_j / \tau)}{\sum_{k \neq i} \exp(z_i^\top z_k / \tau)}
\label{eq:Pij}
\end{equation}
and $S_{ij}$ denotes the soft pairwise target derived from the label vectors:
\begin{equation}
S_{ij} = \frac{s(\mathbf{y}_i, \mathbf{y}_j)}{\sum_{k \neq i} s(\mathbf{y}_i, \mathbf{y}_k)},
\qquad
s(\mathbf{y}_i, \mathbf{y}_j) = \frac{\mathbf{y}_i^\top \mathbf{y}_j}{\|\mathbf{y}_i\| \|\mathbf{y}_j\|}
\label{eq:sij-base}
\end{equation}

\begin{equation}
\mathcal{L}_{\supcon} := - \frac{1}{N}\sum_{i=1}^N \mathcal{L}_{\mathrm{SupCon\text{-}ST}}^{(i)}
\end{equation}

% ...............................................................
\subsubsection{\supconmix}
\label{conosrmix}

\supconmix{} additionally enforces contrast against Mixup-generated virtual samples
$\{\tilde{z}_j\}_{j=1}^N$ that simulate the open space, constructed as in \eqref{eq:mixup}.
The weights are applied to both the augmented batch and the Mixup batch:
\begin{equation}
\mathcal{L}_{\supconmix} := - \frac{1}{N}\sum_{i=1}^N  \mathcal{L}_{\mathrm{SupCon\text{-}ST}}^{(i)}
- \frac{1}{N}\sum_{j=1}^N \mathcal{L}_{\mathrm{SupCon\text{-}ST}}^{\mathrm{Mixup}(j)}
\end{equation}

Note that the two terms are calculated independently: clean anchors are 
contrasted exclusively with clean samples, while Mixup anchors are contrasted 
only with Mixup samples. Unlike the original ConOSR implementation, which 
computes a single loss over the combined union of both batches, we separate 
them. This separation ensures that contrast labels within each term share 
a uniform norm, $1$ for clean samples and $\sqrt{\lambda^2 + (1-\lambda)^2}$ 
for Mixup samples. Consequently, Lemma~\ref{lem:norm} holds for both terms, 
allowing the target to simplify into the inner-product expression in 
\eqref{eq:sij-inner}. Concatenating the batches jointly would introduce 
varying contrast norms across groups, breaking this reduction.

\subsubsection{Integration with \supcon{} and \supconmix{}}
\label{integrationsupcon}
The label smoothing is only applied to the anchor labels, and not to the contrast labels. All the contrast labels $\mathbf{y}_j$ remains one-hot for each anchor label. In this section we will derive the relationship between the smoothing coefficient $\epsilon_i$ and the soft pairwise target $S_{ij}$. We show that the relationship is different depending on if the contrast sample $i$ is a positive (same class label) or negative (different class label)

We first show that this normalization has no effect in Equation~\ref{eq:sij-base}. 
\begin{lemma}[Reduction to inner products]
\label{lem:norm}
Let $\mathbf{y}_i' \in \Delta^{C-1}$, and let all contrast labels have equal norm,
$\|\mathbf{y}_k\| = c > 0$ for all $k \neq i$. If
$\sum_{k \neq i} (\mathbf{y}_i')^\top \mathbf{y}_k > 0$, then
\begin{equation}
S_{ij} = \frac{(\mathbf{y}_i')^\top \mathbf{y}_j}{\sum_{k \neq i} (\mathbf{y}_i')^\top \mathbf{y}_k}.
\label{eq:sij-inner}
\end{equation}
\end{lemma}

\begin{proof}
For every $k \neq i$,
$s(\mathbf{y}_i', \mathbf{y}_k) = (\mathbf{y}_i')^\top \mathbf{y}_k \,/\, (c\,\|\mathbf{y}_i'\|)$.
The factor $1/(c\,\|\mathbf{y}_i'\|)$ is common to the numerator and to every term of the
denominator of \eqref{eq:sij-base}, and cancels.
\end{proof}

\begin{proposition}[Smoothed \supcon{} target]
\label{prop:supcon}
If $\varepsilon_i \in (0,1]$ or $n_i^{+} \geq 1$, then
\begin{equation}
S_{ij} = \frac{(1-\varepsilon_i)\,\mathbbm{1}\{y_j = y_i\} + \varepsilon_i / C}
              {(1-\varepsilon_i)\, n_i^{+} + \varepsilon_i K / C}.
\label{eq:sij-ls}
\end{equation}
\end{proposition}

\begin{proof}
Taking the inner product of \eqref{eq:ls1} with a one-hot contrast label $\mathbf{y}_k$,
and using $\mathbf{y}_i^\top \mathbf{y}_k = \mathbbm{1}\{y_k = y_i\}$ and
$\mathbf{1}^\top \mathbf{y}_k = 1$,
\begin{equation}
(\mathbf{y}_i')^\top \mathbf{y}_k = (1-\varepsilon_i)\,\mathbbm{1}\{y_k = y_i\} + \frac{\varepsilon_i}{C}.
\label{eq:inner}
\end{equation}
Summing over $k \neq i$ gives $(1-\varepsilon_i)\, n_i^{+} + \varepsilon_i K / C$, which is
positive under the stated condition. Lemma~\ref{lem:norm} with $c = 1$ yields \eqref{eq:sij-ls}.
\end{proof}

At $\varepsilon_i = 0$, \eqref{eq:sij-ls} recovers the published target
$S_{ij} = \mathbbm{1}\{y_j = y_i\} / n_i^{+}$; at $\varepsilon_i = 1$ it is uniform,
$S_{ij} = 1/K$; and $\sum_{j \neq i} S_{ij} = 1$ for all $\varepsilon_i$.

\begin{proposition}[Monotone mass transfer]
\label{prop:monotone}
Let $1 \leq n_i^{+} < K$. Increasing $\varepsilon_i$ moves target mass from
positives to negatives: every positive's weight decreases, every negative's
weight increases, and
\begin{equation}
\frac{S_{i,\mathrm{pos}}}{S_{i,\mathrm{neg}}} = 1 + \frac{C\,(1-\varepsilon_i)}{\varepsilon_i}
\label{eq:ratio}
\end{equation}
decreases strictly in $\varepsilon_i \in (0,1]$.
\end{proposition}

\begin{proof}
Both terms in \eqref{eq:sij-ls} share a denominator, so the ratio
\eqref{eq:ratio} is immediate and clearly decreasing. Since each row of $S$
sums to 1, a smaller positive-to-negative ratio means positives lose mass
and negatives gain it.
\end{proof}

In summation, for each anchor, increasing smoothing
therefore steadily shifts target weight $S_{ij}$ from positives to negatives, and decreasing it sharpens the target back toward the positives.

\section{Alignment and Uniformity: Formal Statements and Proofs}
\label{supp:au_decomposition}

\subsection{Decomposition of \sphor{} and \sphormix{}}
The \sphor{} objective \citep[Eqs.~7,~10]{sphor} is
\begin{equation}
\mathcal{L}_{\mathrm{SpHOR}} = \mathcal{L}_{\mathrm{vMF}} + \mathcal{R}_{\mathrm{Ortho}},
\end{equation}
where
\begin{equation}
\mathcal{L}_{\mathrm{vMF}}
= -\frac{1}{N}\sum_{i=1}^{N}\sum_{k=1}^{C} S_{ik}\log P_{ik},
\qquad
P_{ik} = \frac{\exp(\mathbf z_i^\top\boldsymbol\mu_k/\tau)}
{\sum_{j=1}^{C}\exp(\mathbf z_i^\top\boldsymbol\mu_j/\tau)},
\qquad
S_{ik} = \frac{\tilde y_{ik}}{\sum_{j=1}^{C}\tilde y_{ij}},
\end{equation}
\begin{equation}
\mathcal{R}_{\mathrm{Ortho}}
= \log \frac{1}{C(C-1)} \sum_{k\neq l}
\exp\!\Big(\frac{(\boldsymbol\mu_k^\top\boldsymbol\mu_l)^2}{\tau}\Big).
\end{equation}
$\mathcal{R}_{\mathrm{Ortho}}$ depends only on the prototypes
$\{\boldsymbol\mu_k\}_{k=1}^{C}$; it contains neither the targets
$\tilde{\mathbf y}_i$ (hence no $S_{ik}$) nor the projections $\mathbf z_i$.
The targets therefore enter $\mathcal{L}_{\mathrm{SpHOR}}$ only through
$\mathcal{L}_{\mathrm{vMF}}$, which is the target-dependent loss decomposed
in Table~\ref{tab:align_unif}. Expanding the log in $\mathcal{L}_{\mathrm{vMF}}$,
\begin{align*}
\mathcal{L}_{\mathrm{vMF}}
&= -\frac{1}{N}\sum_{i=1}^{N}\left(
\sum_{k=1}^{C} S_{ik}\,\frac{\mathbf z_i^\top\boldsymbol\mu_k}{\tau}
-\Big(\sum_{k=1}^{C} S_{ik}\Big)
\log\sum_{j=1}^{C}\exp\!\Big(\frac{\mathbf z_i^\top\boldsymbol\mu_j}{\tau}\Big)
\right) \\
&= \underbrace{-\frac{1}{N}\sum_{i=1}^{N}\sum_{k=1}^{C}
S_{ik}\,\frac{\mathbf z_i^\top\boldsymbol\mu_k}{\tau}}_{\text{Alignment}}
\;+\;
\underbrace{\frac{1}{N}\sum_{i=1}^{N}
\log\sum_{j=1}^{C}\exp\!\Big(\frac{\mathbf z_i^\top\boldsymbol\mu_j}{\tau}\Big)}_{\text{Uniformity}},
\end{align*}
using $\sum_{k=1}^{C} S_{ik}=1$, which holds by construction of $S_{ik}$;
this matches \citet[Eq.~9]{sphor}. Since
$\mathbf z_i,\boldsymbol\mu_k\in\mathbb S^{p-1}$,
$\mathbf z_i^\top\boldsymbol\mu_k = 1-\tfrac12\|\mathbf z_i-\boldsymbol\mu_k\|_2^2$.
Substituting, the constants $-1/\tau$ and $+1/\tau$ cancel:
\begin{equation}
\mathcal{L}_{\mathrm{vMF}}
= \frac{1}{N}\sum_{i=1}^{N}\left(
\frac{1}{2\tau}\sum_{k=1}^{C} S_{ik}\,\|\mathbf z_i-\boldsymbol\mu_k\|_2^2
+\log\sum_{j=1}^{C}\exp\!\Big(-\frac{\|\mathbf z_i-\boldsymbol\mu_j\|_2^2}{2\tau}\Big)
\right),
\end{equation}
which is the prototype row of Table~\ref{tab:align_unif}. Hence
\begin{equation}
\mathcal{L}_{\mathrm{SpHOR}}
= \frac{1}{N}\sum_{i=1}^{N}\Big(\mathcal{L}^{(i)}_{\mathrm{align}}
+ \mathcal{L}^{(i)}_{\mathrm{unif}}\Big) + \mathcal{R}_{\mathrm{Ortho}},
\end{equation}
where only $\mathcal{L}^{(i)}_{\mathrm{align}}$ depends on the targets.
For \sphormix{}, the batch additionally contains the mixup samples
\citep[Eq.~3]{sphor}; the derivation is unchanged.

\subsection{Decomposition of \supcon{} and \supconmix{}}
\begin{equation}
\mathcal{L}_{\mathrm{scst}}
= -\frac{1}{N}\sum_{i=1}^{N}\sum_{j\neq i} S_{ij}\log P_{ij},
\qquad
P_{ij} = \frac{\exp(\mathbf z_i^\top\mathbf z_j/\tau)}
{\sum_{n\neq i}\exp(\mathbf z_i^\top\mathbf z_n/\tau)},
\qquad
S_{ij} = \frac{s(\tilde{\mathbf y}_i,\tilde{\mathbf y}_j)}
{\sum_{n\neq i} s(\tilde{\mathbf y}_i,\tilde{\mathbf y}_n)} .
\end{equation}
Expanding the log,
\begin{align*}
\mathcal{L}_{\mathrm{scst}}
&= -\frac{1}{N}\sum_{i=1}^{N}\left(
\sum_{j\neq i} S_{ij}\,\frac{\mathbf z_i^\top\mathbf z_j}{\tau}
-\Big(\sum_{j\neq i} S_{ij}\Big)
\log\sum_{n\neq i}\exp\!\Big(\frac{\mathbf z_i^\top\mathbf z_n}{\tau}\Big)
\right) \\
&= \underbrace{-\frac{1}{N}\sum_{i=1}^{N}\sum_{j\neq i}
S_{ij}\,\frac{\mathbf z_i^\top\mathbf z_j}{\tau}}_{\text{Alignment}}
\;+\;
\underbrace{\frac{1}{N}\sum_{i=1}^{N}
\log\sum_{n\neq i}\exp\!\Big(\frac{\mathbf z_i^\top\mathbf z_n}{\tau}\Big)}_{\text{Uniformity}},
\end{align*}
using $\sum_{j\neq i} S_{ij}=1$, which holds by construction of $S_{ij}$;
this matches \citet[Eq.~17]{conosr}. Since
$\mathbf z_i\in\mathbb S^{p-1}$,
$\mathbf z_i^\top\mathbf z_j = 1-\tfrac12\|\mathbf z_i-\mathbf z_j\|_2^2$.
Substituting, the constants $-1/\tau$ and $+1/\tau$ cancel:
\begin{equation}
\mathcal{L}_{\mathrm{scst}}
= \frac{1}{N}\sum_{i=1}^{N}\left(
\frac{1}{2\tau}\sum_{j\neq i} S_{ij}\,\|\mathbf z_i-\mathbf z_j\|_2^2
+\log\sum_{n\neq i}\exp\!\Big(-\frac{\|\mathbf z_i-\mathbf z_n\|_2^2}{2\tau}\Big)
\right),
\end{equation}
which is the batch row of Table~\ref{tab:align_unif}. For \supconmix{},
the batch additionally contains the mixup samples, whose targets are
the mixed labels \citep[Eq.~3]{conosr}; the derivation is unchanged.

\end{document}